%% file: paper.tex
\documentclass[onefignum,onetabnum]{siamonline190516}

\usepackage{graphicx}
\usepackage{lineno}
\usepackage{amsfonts}
\usepackage{algorithm}
\usepackage{algorithmic}
\usepackage{lipsum}
\usepackage{epstopdf}
\ifpdf
  \DeclareGraphicsExtensions{.eps,.pdf,.png,.jpg}
\else
  \DeclareGraphicsExtensions{.eps}
\fi

\usepackage{enumitem}
\setlist[enumerate]{leftmargin=.5in}
\setlist[itemize]{leftmargin=.5in}
\include{snippets}
\graphicspath{ {./src/} }
\begin{document}

% \begin{frontmatter}
\title{Harmonic Beltrami Signature: A Novel 2D Shape Representation for Object Classification}

\author{
    Chenran Lin\thanks{Department of Mathematics, The Chinese University of Hong Kong, Hong Kong (\email{crlin@math.cuhk.edu.hk})}
\and
    Lok Ming Lui\thanks{Department of Mathematics, The Chinese University of Hong Kong, Hong Kong (\email{lmlui@math.cuhk.edu.hk})}
}
\headers{Harmonic Beltrami Signature}{Chenran Lin, and Lok Ming Lui}

\maketitle

\begin{abstract}
  There is a growing interest in shape analysis in recent years. We present a novel shape signature for 2D bounded simply-connected domains, named the Harmonic Beltrami signature (HBS). The proposed signature is based on the harmonic extension of the conformal welding map of a unit circle and its Beltrami coefficient. We show that there is a one-to-one correspondence between the quotient space of HBS and the space of 2D simply-connected shapes up to a translation, rotation and scaling. With a suitable normalization, each equivalence class in the quotient space of HBS is associated to a unique representative. It gets rid of the conformal ambiguity. As such, each shape is associated to a unique HBS. Conversely, the associated shape of a HBS can be reconstructed based on quasiconformal Teichm\"uller theories, which is uniquely determined up to a translation, rotation and scaling. The HBS is thus an effective fingerprint to represent a 2D shape. The robustness of HBS is studied both theoretically and experimentally. With the HBS, simple metric, such as $L^2$, can be used to measure geometric dissimilarity between shapes. Experiments have been carried out to classify shapes in different classes using HBS. Results show good classification performance, which demonstrate the efficacy of our proposed shape signature.
\end{abstract}


\section{Introduction}
\label{intro}

% Working on the introduction

    The outline of a shape contains important information, which can be used in many applications, such as medical image analysis, image segmentation, recognition, registration and so on. In order to utilize the shape information, an effective descriptor to represent a shape is a necessary and fundamental tool for many applications in pattern recognition and computer visions. Nevertheless, defining a robust shape signature to describe the space of shapes is still a mathematically challenging problem. A good shape signature should be easy to compute and retain essential geometric features of a shape. Meanwhile, a practical shape signature should be invariant under rigid motion (rotation, translation and scaling) and robust to noise. More desirably, the space of shape signatures should inherits a natural and simple metric. With the natural metric, two shapes can be quantitatively compared and prior shape information can be incorporated into various imaging models by adding a penalty term to further improve the accuracy. Of course, the shape signature must be simple to manipulate such that the modified imaging model is numerically manageable.
    
    Because of its significance, this problem has been widely studied and different models to build the metric shape space have been proposed. In general, existing approaches for constructing shape descriptors can be divided into two main categories, namely, the region-based methods and the contour-based methods. Region-based methods use all information of pixels within a shape. While more information is considered, these methods are often more computationally demanding. In contrast, the contour-based methods only use the boundary information of the shape. Various descriptors capturing the essential geometric features of the shape contour have been recently proposed. It is worth mentioning that many of these shape descriptors in either categories cannot completely capture the geometric information of the shape. In other words, the shape associated to a given shape descriptor cannot be uniquely determined up to rigid motions. Motivated by this, we propose in this paper a shape signature for 2D bounded simply-connected shapes, called the Harmonic Beltrami signature, based on the quasiconformal Teichm\"uller theories. Each shape signature is associated to a unique shape up to rigid motions. Thus, the proposed signature can capture the geometric information of a shape completely.
    
    %Nevertheless, it is a big challenge to define a robust descriptor for the space of shapes, even for the simplest situation of 2D simply-connected objects. In order to manipulate shapes and utilize their geometric information, there is an urgent demand to find a simple and robust representation to mathematically describe shape contours. The shape representation should also inherits a natural metric to measure geometric differences between two shape representations. With the geometric dissimilarity metric, two shapes can be quantitatively compared and prior shape information can be incorporated into various imaging models by adding a penalty term to further improve the accuracy. Due to its significance, this problem has been widely studied and different models to build the metric shape space have been proposed.

    More specifically, given a simply-connected bounded domain, the conformal disk parameterizations of the inner and outer regions are computed. It gives rise to the conformal welding of the boundary contour of the domain. The harmonic extension of the welding map can be computed, whose Beltrami coefficient then defines the shape signature, called the {\it Harmonic Beltrami signature (HBS)}. Theoretically, it can be shown that there is a one-to-one correspondene between the quotient space of HBS and the space of 2D simply-connected shapes up to rigid motions. With a suitable normalization, each equivalence class in the quotient space of HBS has a unique representative, which helps to get rid of the conformal ambiguity. In particular, each shape is associated to a unique HBS. Also, given a HBS, the associated shape can be uniquely reconstructed up to a translation, rotation and scaling. As such, the HBS can be regarded as an effective fingerprint to represent a 2D shape. Note that conformal map of a region is robust to noises on the boundary contour of the domain. The proposed HBS, which is based on the conformal maps and the harmonic extension, is robust to noises. The proposed signature also allows us to study the space of shapes by analyzing the space of HBS, which is easy to manipulate. In fact, with the HBS, simple metric can be used to measure geometric dissimilarity between shapes. This can be applied to various pattern recognition and image analysis tasks.

    The paper is organized as follows: Section \ref{related work} reviews some related topics about shape descriptors; Section \ref{background} introduces some theoretic background; Section \ref{main} explains our proposed Harmonic Beltrami signature in details; Section \ref{implementation} gives the implementation details; Section \ref{result} reports our experimental results. The paper is concluded in Section \ref{conclusion} and we point out several future directions.

    \begin{figure}
        \begin{center}
        \includegraphics[width=11cm]{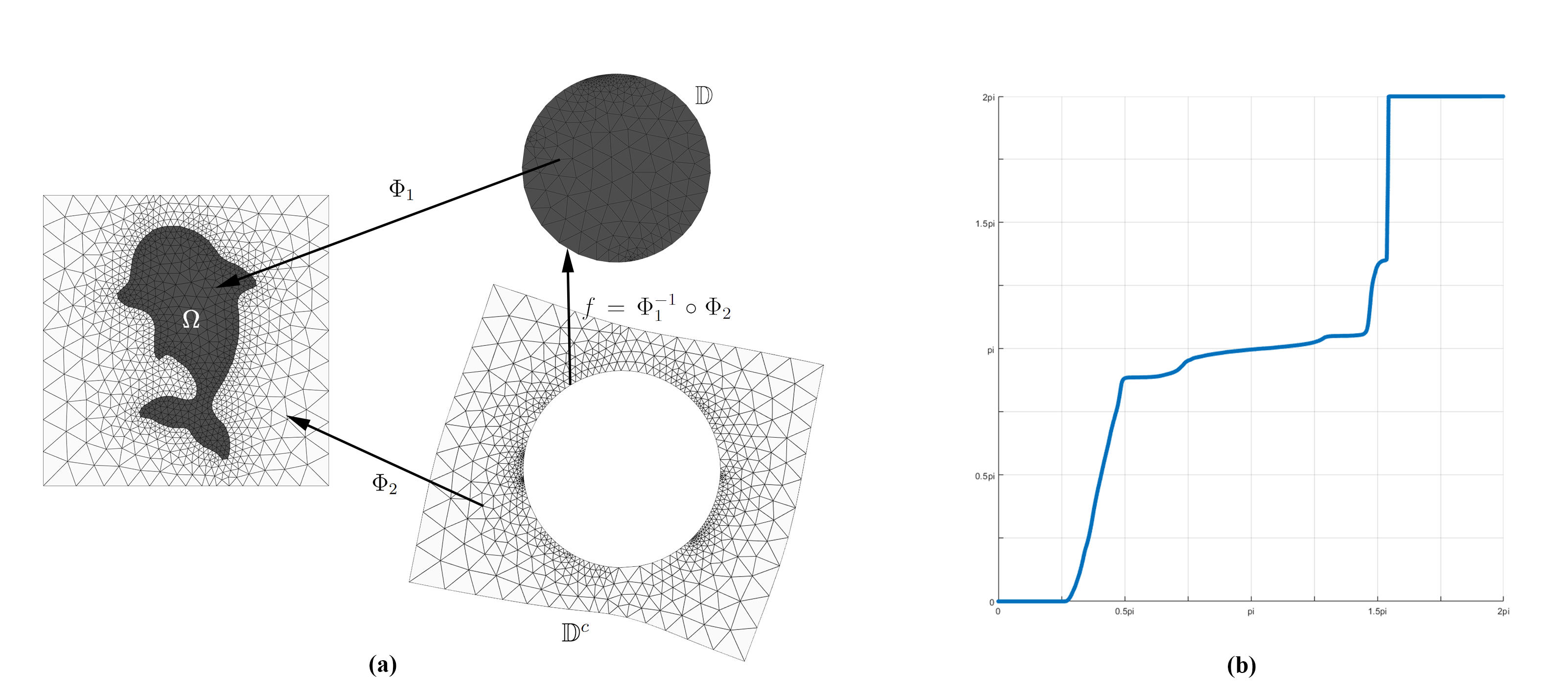}
        \caption{(a) The Illustration of how conformal welding $f$ is defined from given 2D simply-connected domain $\Omega$; (b) The image of conformal welding $f: [0, 2\pi) \rightarrow [0, 2\pi)$}
        \end{center}
        \label{fig1}
    \end{figure}

\section{Contributions}
The contributions of this paper can be summarized as follows.
\begin{enumerate}
    \item Firstly, we propose a new shape signature, called the Harmonic Beltrami signature, to effectively represent 2D simply-connected shapes. Every shape has a unique Harmonic Beltrami signature. Conversely, given a Harmonic Beltrami signature, its corresponding shape can be determined up to a translation, rotation and scaling.
    \item Secondly, the proposed Harmonic Beltrami signature solves the issue of conformal ambiguities facing the conformal welding signature.
    \item Thirdly, we propose a practical procedure to normalize the Harmonic Beltrami signature to handle the non-uniqueness issue, with rigorous theoretical justifications.
    \item Fourthly, we propose a reconstruction algorithm to construct the corresponding shape from the Harmonic Beltrami signature up to a rotation, translation and scaling. This allows us to go back and fro between shapes and Beltrami signatures in the imaging model.
    \item Finally, the proposed Harmonic Beltrami signature inherits a simple metric, namely, the $L^2$ distance, to measure the geometric dissimilarity between shapes. We have applied the shape distance to shape classification and shown satisfactory results. 
\end{enumerate}

\section{Related works}\label{related work}
    Shape representation and description is an enduring field and there have been extensive and in-depth discussions in the past several decades. Demisse \etal \cite{demisse2017deformation} proposed a method to represent an ordered set of points sampled from a curved shape as an element of a finite dimensional matrix Lie group. Mokhtarian \etal \cite{mokhtarian1997efficient} used the maxima of curvature zero-crossing contours of Curvature Scale Space image to represent the shapes of object boundary contours. Lui \etal \cite{lui2013shape} extracted each component of a 2D multi-connected shape, then the conformal weldings represent all components and conformal modules describe relationships between components. 
    
    A more meticulous survey about shape representations can be found in \cite{zhang2004review}. Generally speaking, all of these representation techniques can be divided into two major categories, \textit{contour-based} methods and \textit{region-based} methods, depending on whether shape features are extracted from the contour only or from the whole shape region.

    \subsection{Contour-based methods}
        As its name suggests, this kind of representations only exploits the information providing by shape boundary. A very natural idea is that the boundary can be taken as a whole, from which a multi-dimensional numeric feature vector can be calculated and becomes the demanded representation.

        The simplest features are area, circularity, curvature and so on and their combination can be used as shape representation. Peura \etal \cite{peura1997efficiency} proposed such a descriptor including convexity, ratio of principle axis, circular variance and elliptic variance. Belongie \etal \cite{belongie2002shape} tried in a different way and built a representation based on Hausdorff distance, called shape context. For any boundary point $p$, they calculated the Hausdorff distance $d_{pq}$ and the orientation $\theta_{pq}$ with any other boundary point $q$, then these $d_{pq}$ and $\theta_{pq}$ are quantized to create a histogram map $H_p$, which is used to represent the point $p$. All the histograms $H_p$ are flattened and concatenated to form the context of the shape. Asada \etal \cite{asada1986curvature} attempted to smooth the boundary by Gaussian filter as well as the second derivatives of Gaussian filter, then the remained inflection points are expected to be significant object characteristics.

        Some other contour-based representations pay more attention to local boundary information and break the shape down into many pieces. Chain code describes an object by a sequence of unit-size line segments with a given orientation, which was introduced by Freeman \etal \cite{freeman1961encoding}. Groskey \etal \cite{grosky1990index} proposed polygon decomposition as representation. The given shape boundary is broken down into line segments by polygon approximation. The feature for each segment is expressed as four elements, internal angle, distance from the next vertex, and its $x$ and $y$ coordinates. Berretti \etal \cite{berretti2000retrieval} extended Groskeyet's model. The curvature zero-crossing points from a Gaussian smoothed boundary are used to obtain smooth curve, called tokens. The features for each token are its maximum curvature and orientation, and the similarity between two tokens is measured by the weighted Euclidean distance. 
    
    \subsection{Region-based methods}
        Different from the previous category, region-based representations make the best use of all the pixels within the given shape region. Geometric moment is a classical and representative region-based shape description with form
        \begin{equation*}
            m_{pq} = \sum_x \sum_y x^p y^q f(x, y),
        \end{equation*}
        where $p, q = 0, 1, 2, \cdots$ and $f$ is the given shape. Hu published the first significant paper about geometric moment and applied it in pattern recognition \cite{hu1962visual}. Taubin \etal \cite{taubin1991object,taubin1991recognition} proposed algebraic moment, which is computed from the first $m$ central moments and is given as the eigenvalues of predefined matrices $M_{j, k}$, whose elements are scaled factors of the central moments. Zhang \etal \cite{zhang2002generic} proposed Generic Fourier descriptor which is acquired by applying a 2D Fourier transform on a polar-raster sampled image
        \begin{equation*}
            PF_2(\rho, \phi) = \sum_r \sum_k f(r, \theta_k) e^{2 \pi i (\frac{r}{R} \rho + \theta_k \phi)},
        \end{equation*}
        where $0 \le r < R$, $0 \le \rho < R$, $0 \le k < T$, $0 \le \phi < T$, $\theta_k = \frac{2\pi k}{T}$ and $R, T$ are the radial frequency resolution and angular frequency resolution respectively

\section{Theoretical basis}\label{background}
    \subsection{Quasi-conformal mapping and Beltrami equation}
        % Let $f: \Omega \subset \C \rightarrow \C$ be .
        % The following differential operators are more convenient for discussion
        % \begin{equation*}
        %     \frac{\partial}{\partial z} := \frac{1}{2}(\Part{}{x} - i \Part{}{y}), \Part{}{\overline{z}}:= \frac{1}{2}(\Part{}{x}+i \Part{}{y})
        % \end{equation*}
        A complex function $f: \Omega \subset \C \rightarrow \C$ is said to be \textit{quasi-conformal} associated to $\mu$ if $f$ is orientation-preserving and satisfies the following \textit{Beltrami equation}:
        \begin{equation}\label{beltrami eq}
            \Part{f}{\overline{z}} = \mu(z) \Part{f}{z}
        \end{equation}
        where $\mu(z)$ is a complex-valued Lebesgue measurable function satisfying $\norm{\mu}_\infty < 1$. More specifically, this $\mu: \Omega \rightarrow \D$ is called the \textit{Beltrami coefficient} of $f$
        \begin{equation}\label{mu def}
            \mu = \frac{f_{\overline{z}}}{f_{z}}
        \end{equation}
        
        In terms of the metric tensor, consider the effect of the pullback under $f$ of the Euclidean metric $ds^2_E$, the resulting metric is given by:
        \begin{equation}
            f^*(ds^2_E) = \abs{\Part{f}{z}}^2 \abs{dz + \mu(z)d\overline{z}}^2
        \end{equation}
        which, relative to the background Euclidean metric $dz$ and $d\overline{z}$, has eigenvalue $(1+\abs{\mu})^2 \abs{\Part{f}{z}}^2$ and $(1-\abs{\mu})^2 \abs{\Part{f}{z}}^2$. 

        Therefore, inside the local parameter domain around some point $p$, $f$ can be considered as a map composed of a translation to $f(p)$ together with the multiplication of a stretch map $S(z) = z + \mu(p)\overline{z}$ and conformal function $f_z(p)$, which may be expressed as follows:
        \begin{equation}\label{local f}
            f(z) =  f(p)+S(z)f_z(p) = f(p)+(z+\mu(p)\overline{z})f_z(p).
        \end{equation}
        $S(z)$ makes $f$ map a small circle to a small ellipse and all the conformal distortion of $f$ is caused by $\mu$. To form $\mu(p)$, we can determine the angles of the directions of maximal magnification and shrinkage and the amount of them as well. Specially, the angle of maximal magnification is $\arg(\mu(p))/2$ with magnifying factor $1+\abs{\mu(p)}$; the angle of maximal shrinkage is the orthogonal angle $\arg(\mu(p))/2 - \pi/2$ with shrinkage factor $1-\abs{\mu(p)}$. The distortion or dilation is given by:
        \begin{equation}
            K = \frac{1+\abs{\mu(p)}}{1-\abs{\mu(p)}}.
        \end{equation}
        Thus, the Beltrami coefficient $\mu$ gives us important information about the properties of the map (see figure \ref{fig3}) and $\mu$ is a measure of non-conformality. In particular, the map $f$ is conformal around a small neighborhood of $p$ when $\mu(p)=0$ and if $\mu(z)=0$ everywhere on $\Omega$, $f$ us called \textit{conformal} or \textit{holomorphic} on $\Omega$.

        \begin{figure}
            \begin{center}
                \includegraphics[width=7.6cm]{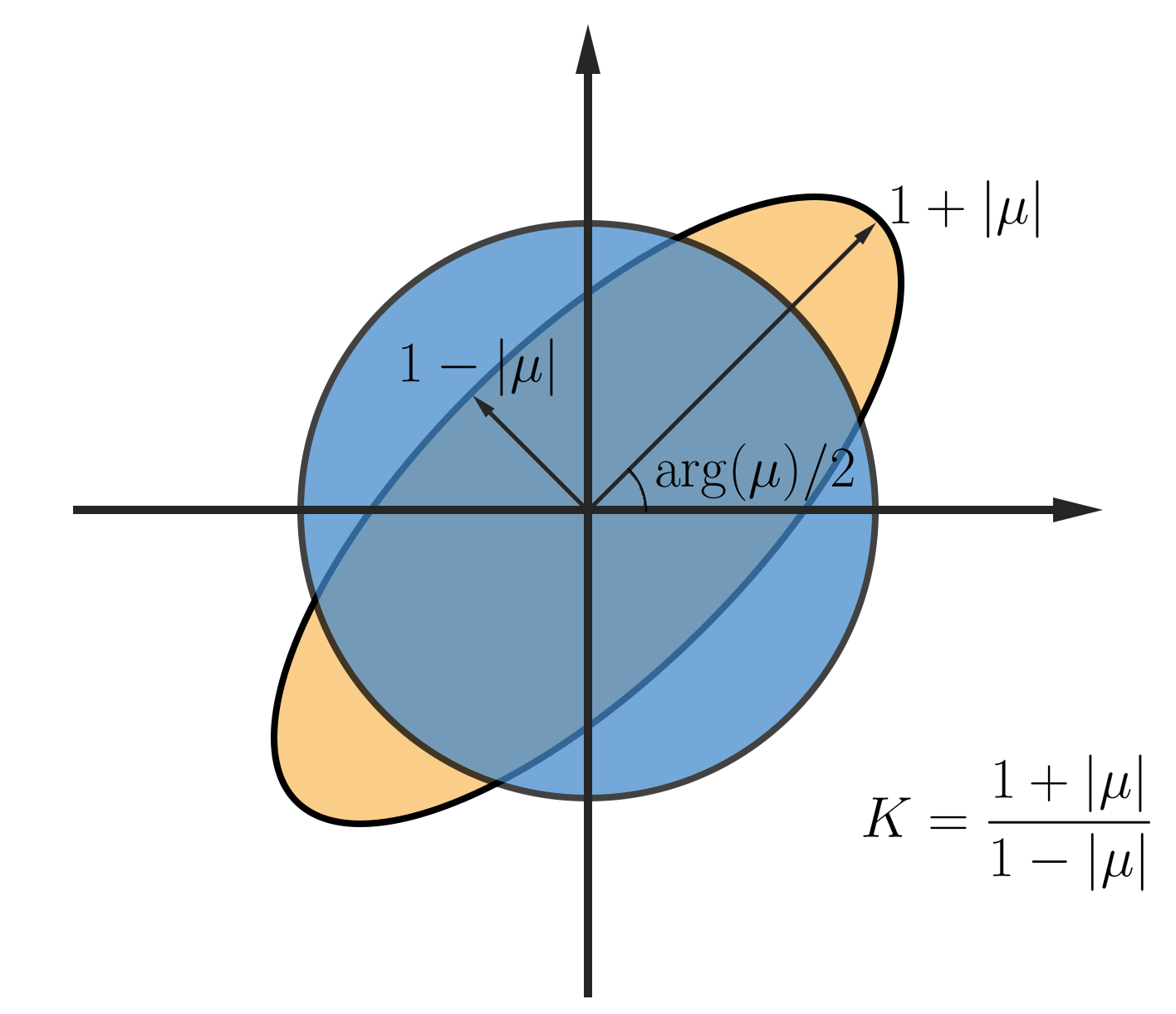}
            \end{center}
            \caption{Quasi-conformal maps infinitesimal circles to ellipses. The Beltrami coefficient measure the distortion or dilation of the ellipse under the QC map.}
            \label{fig3}
        \end{figure}

        Note that there is a one-to-one correspondence between the quasi-conformal mapping $f$ and its Beltrami coefficient $\mu$. Given $f$, there exists a Beltrami coefficient $\mu$ satisfying the Beltrami equation by equation (\ref{mu def}). Conversely, the following theorem states that given an admissible Beltrami coefficient $\mu$, there always exists an quasi-conformal mapping $f$ associating with this $\mu$.

        \begin{theorem}[Measurable Riemannian Mapping Theorem]\label{Measurable Riemannian Mapping Theorem}
            Suppose $\mu : \C \rightarrow \C$ is Lebesgue measurable satisfying $\norm{\mu}_\infty <1$; then, there exists a quasi-conformal homeomorphism $f$ from $\C$ onto itself, which is in the Sobolev space $W_{1,2}(\C)$ and satisfies the Beltrami equation in the distribution sense. The associated quasi-conformal homeomorphism $f$ is unique up to a Mobi\"us transformation. Furthermore, by fixing $0$, $1$ and $\infty$, the $f$ is uniquely determined.
        \end{theorem}

        Suppose $f, g: \C \rightarrow \C$ are complex-valued function with Beltrami coefficient $\mu_f, \mu_g$ respectively. Then the Beltrami coefficient for the composition $g \circ f$ is given by 
        \begin{equation}\label{mu of composition}
            \mu_{g \circ f} = \frac{\mu_f+(\mu_g \circ f) \tau}{1+\overline{\mu_f}(\mu_g \circ f) \tau},
        \end{equation}
        where $\tau = \frac{\overline{f_z}}{f_z}$. Note that when $g$ is conformal, $\mu_g = 0$ and 
        \begin{equation}\label{mu of conformal composition}
            \mu_{g \circ f} = \mu_f.
        \end{equation}

    \subsection{Conformal welding}\label{welding}
        Given a 2D bounded simply-connected shape, we can treat it as a 2D bounded simply-connected domain $\Omega \subset \C$, by Riemann mapping theorem, there exist conformal functions $\Phi_1: \D \rightarrow \Omega$ and $\Phi_2: \D^c \rightarrow \Omega^c$. $\Phi_1$ and $\Phi_2$ are unique up to a \textit{Mobi\"us transformation}:
        \begin{equation}
            M(z) = e^{i \theta} \frac{z - a}{1 - \overline{a}z}.
        \end{equation}
        Then we can define \textit{conformal welding} as:
        \begin{equation}
            f = \Phi_1^{-1} \circ \Phi_2.
        \end{equation}
        Such $f: \partial \D \rightarrow \partial \D$ is a diffeomorphism from $\partial \D$ to itself, which can be also thought as a periodic real-valued monotone increasing function $f_\R: [0, 2\pi) \rightarrow [0, 2\pi)$ such that $f(e^{i \theta}) = e^{i f_\R(\theta)}$ (see figure \ref{fig1}).
        
        However, such welding mappings are not unique because of the arbitrariness of Riemann mappings, as shown in figure \ref{unsatble}.
        
        \begin{figure}
            \begin{center}
                \includegraphics[width=8cm]{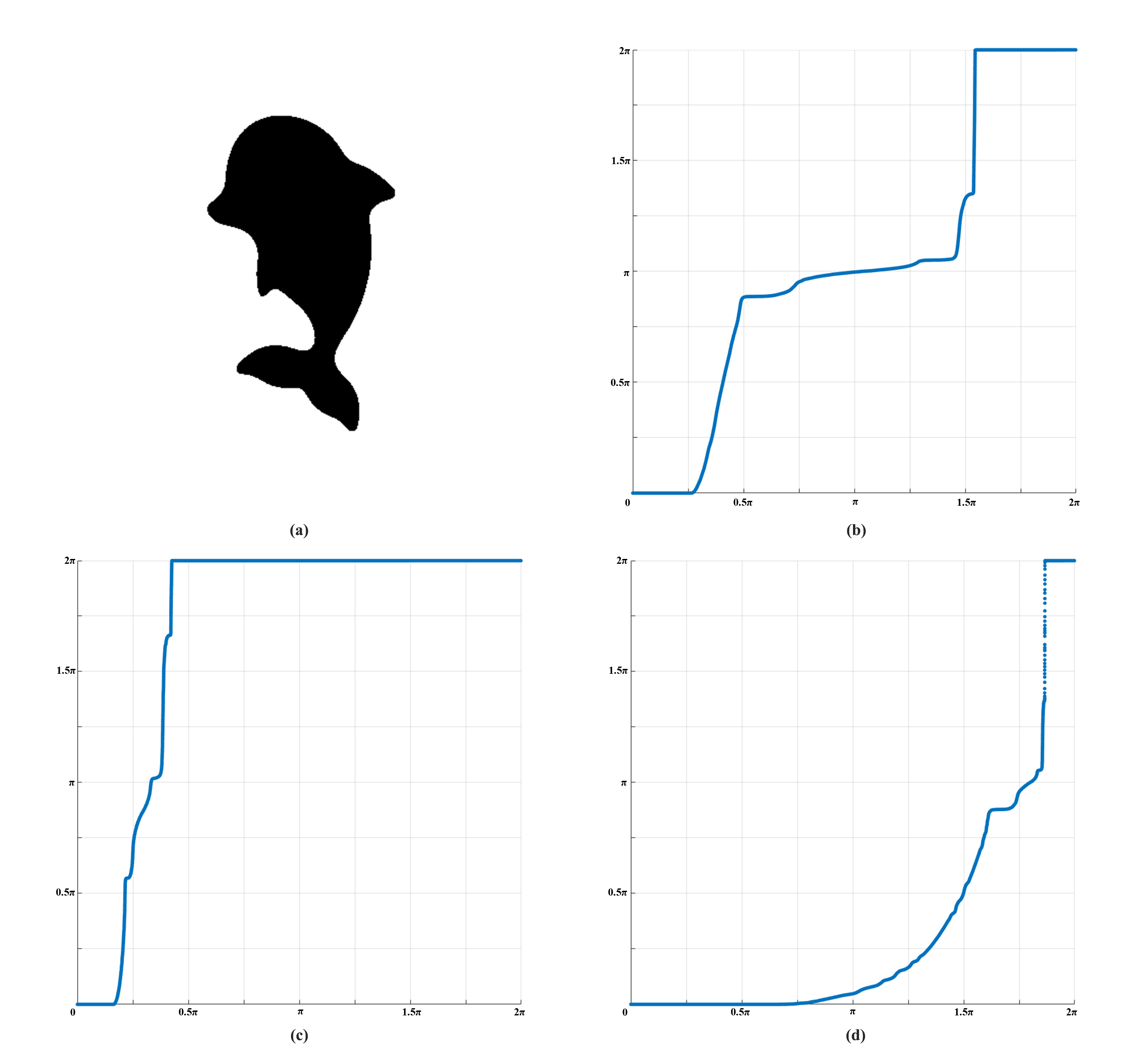}
            \end{center}
            \caption{Different conformal welding mappings (b), (c) and (d) of the same shape (a)}
            \label{unsatble}
        \end{figure}

    \subsection{Harmonic function and Poisson integral}
        A complex-valued function $f$ defined on $\Omega \subset \C$ is called \textit{harmonic} if it satisfies the \textit{Laplace's equation}:
        \begin{equation}
            \Delta f = 4 \frac{\partial^2 f}{\partial z \partial \overline{z}} = \frac{\partial^2 f}{\partial x^2} + \frac{\partial^2 f}{\partial y^2} = 0,
        \end{equation}
        where $z = x + iy$, $\bar{z} = x -iy$.

        Chen \etal \cite{chen2010compositions} proved following theorem, which tells us in what condition the composition of harmonic mappings and other mappings can inherit the harmonicity.
        \begin{theorem}\label{composition of harmonic}
            Let $f$ be a harmonic mapping, $f \circ g$ is harmonic if and only if $g(z) = az + b \overline{z} + c$, where $a$, $b$ and $c$ are constants and $g \circ f$ is harmonic if and only if $g$ is analytic or anti-analytic.
        \end{theorem}
        
        The harmonic function on a compact set is determined by its restriction to the boundary, which follows from the maximum principle, and the progress to find a harmonic function from the given domain and the value in domain's boundary is call \textit{Dirichlet problem}. For a special case, where the domain is unit disk, \textit{Poisson integral} shows a method to obtain the solution $H : \overline{\D} \rightarrow \C$ of Dirichlet problem from a continuous $f$ on $\partial \D$
            \begin{equation}\label{poisson integral}
                H(re^{i\theta}) = \frac{1}{2\pi}\int_0^{2\pi} \frac{(1-r^2)f(e^{i \varphi})}{1 - 2 r cos (\varphi - \theta) + r^2} d\varphi.
            \end{equation}
        Such $H$ is harmonic on $\D$ and continuous on $\overline{\D}$ and has the same value with $f$ on the $\partial \D$, i.e. $H(e^{i\theta}) = f(e^{i\theta})$ (see figure \ref{harmonic}). Of course, it is uniquely determined by $f$.
        
        \begin{figure}
            \begin{center}
                \includegraphics[width=12cm]{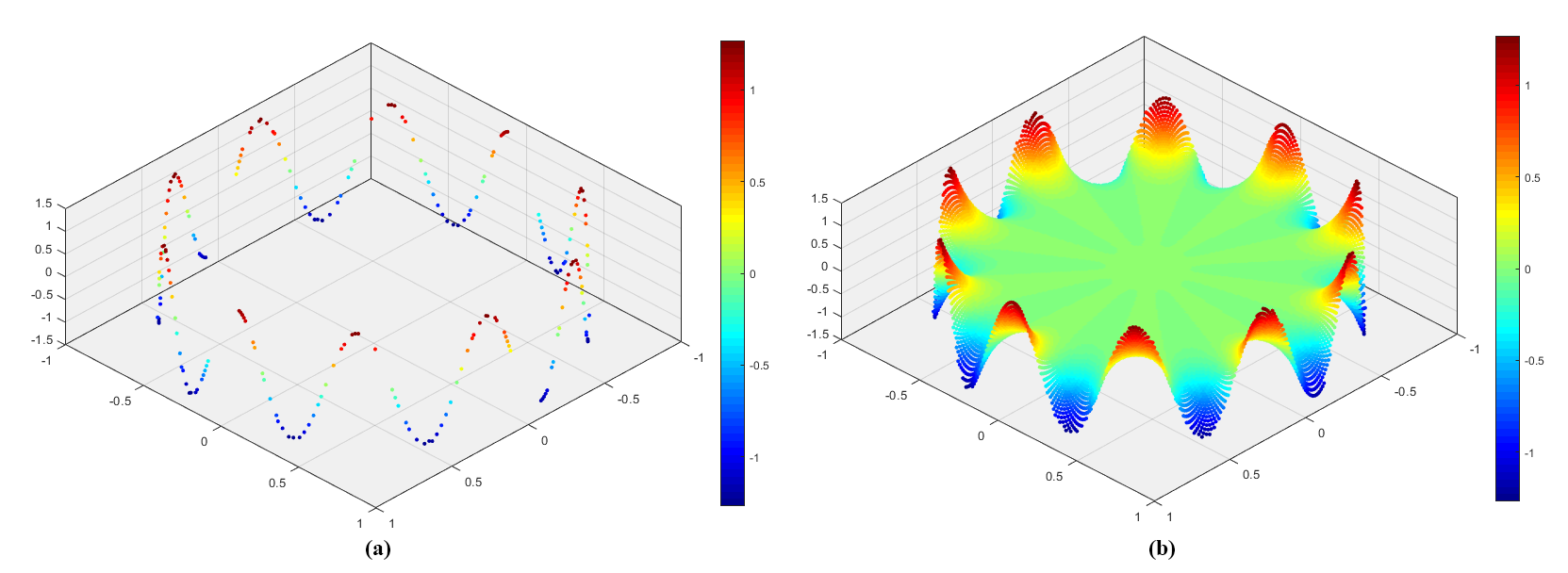}
            \end{center}
            \caption{(a) Continuous function $f(e^{i \theta}) = \sin(10 \theta) + \cos(10 \theta)$ defined on $\partial \D$; (b) The corresponding harmonic function $H$ generated from $f$ by the equation (\ref{poisson integral}). Note that we used real-valued function to illustrate the progress of harmonic extension for the convenience, but it is also feasible for complex-valued function.}
            \label{harmonic}
        \end{figure}

\section{Harmonic Beltrami signature (HBS)}\label{main}
In this chapter, we describe our proposed shape signature, called the {\it Harmonic Beltrami signature (HBS)}, to represent a 2D bounded simply-connected domain $\Omega$. The space of HBS inherits a natural metric, so that geometric distance between two shapes can be easily measured. In the following sections, the definition of HBS and some of its theoretical analysis are addressed.

\subsection{Definition of Harmonic Beltrami Signature}
Consider a bounded simply-connected domain $\Omega\subset \mathbb{C}$. Suppose $\Omega$ is a quasicircle, which is the image of the unit disk under a quasiconformal map. Let $f = \Phi_1^{-1} \circ \Phi_2$ be the conformal welding of $\Omega$, where  $\Phi_1: \D \rightarrow \Omega$ and $\Phi_2: \D^c \rightarrow \Omega^c$ are the conformal mappings. Denote the harmonic extension of $f$ as $H:\mathbb{D}\to \mathbb{D}$ by equation (\ref{poisson integral}).

\begin{definition}
The {\it Harmonic Beltrami Signature (HBS)} is a complex-valued function $B:\mathbb{D}\to \mathbb{D}$ with $||B||_{\infty}<1$ defined as
\begin{equation}
        B:= \mu_H = \frac{H_{\overline{z}}}{H_z}.
\end{equation}
\end{definition}

Note that the HBS is not unique without suitable normalization. According to Riemann mapping theorem, the conformal mappings $\Phi_1: \D \rightarrow \Omega$ and $\Phi_2: \D^c \rightarrow \Omega^c$ are not unique. Suppose $\tilde{\Phi}_1 = \Phi_1 \circ M_1$, $\tilde{\Phi}_2 = \Phi_2 \circ M_2$, where $M_1, M_2$ are Mobi\"us transformations, the corresponding conformal welding is
\begin{equation}\label{tilde f}
    \tilde{f} = \tilde{\Phi}_1^{-1} \circ \tilde{\Phi}_2 = M_1^{-1} \circ \Phi_1^{-1} \circ \Phi_2 \circ M_2 = M_1^{-1} \circ f \circ M_2.
\end{equation}
Therefore, the harmonic extension and hence the HBS are not unique due to conformal ambiguities. This motivates us to give the following definition of equivalence.

\begin{definition}
Two HBS $B$ and $\tilde{B}$ are said to be {\it equivalent} if $B=\mu_H$ and $\tilde{B} = \mu_{\tilde{H}}$, where $H$ and $\tilde{H}$ are respectively the harmonic extensions of a diffeomorphism $f:\mathbb{S}^1\to \mathbb{S}^1$ and $\tilde{f} = M_1^{-1} \circ f \circ M_2$ for some Mobi\"us transformations $M_1$ and $M_2$. In this case, we denote $B \sim \tilde{B}$. Also, the equivalence class of $B$ is denoted by $[B]$.
\end{definition}

In this work, we consider the quotient space of HBS $\mathcal{B} = \{B:\mathbb{D}\to \mathbb{D}:B \text{ is a HBS} \} \,/ \sim$ to study the quotient space of shapes $\mathcal{S} = \{\Omega \subset \C : \Omega \text{ is bounded simply-connected}\} \, / \approx$, where $\Omega \approx \bar{\Omega}$ iff $\bar{\Omega} = F(\Omega)$ and $F$ is composed of translation, rotation and scaling. The following theorem illustrates that HBS is indeed an effective representation.

\begin{theorem}\label{one to one equivalence class}
There is a one-to-one correspondence between $\mathcal{B}$ and $\mathcal{S}$. In particular, given $[B]\in \mathcal{B}$, its associated shape $\Omega$ can be determined up to a Mobi\"us transformation. Also, if $\Phi_2$ is chosen such that $\Phi_2(\infty) = \infty$, $\Omega$ is determined up to a translation, rotation and scaling.
\end{theorem}

\begin{proof}
Given $\Omega$, there exists a unique $[B]\in \mathcal{B}$ corresponding to $\Omega$ by the definition of equivalence class of HBS. Conversely, let $[B]\in \mathcal{B}$ and $B$ is a HBS in $[B]$. Define $\mu:\mathbb{C}\to \mathbb{C}$ as
\begin{equation}\label{reconstructionmu}
\mu := \begin{cases}
B \text{ on }\mathbb{D}\\
0 \text{ on }\mathbb{D}^c.
\end{cases}
\end{equation}

According to Measurable Riemannian Mapping Theorem \ref{Measurable Riemannian Mapping Theorem}, there exists $G:\mathbb{C}\to \mathbb{C}$ such that $G_{\bar{z}}/G_z = \mu$. $G$ is unique up to a Mobi\"us transformation. In other words, if $G_1$ and $G_2$ are two quasiconformal maps satisfying the above requirement, then $G_2=M\circ G_1$, where $M$ is a Mobi\"us transformation. Let $\Omega = G(\mathbb{D})$, we claim that the HBS of $\Omega$ is $B$. To see this, let $\Phi_1: \mathbb{D}\to \Omega$ be the conformal parameterization of $\Omega$. By construction, $G|_{\mathbb{D}^c}: \mathbb{D}^c\to \Omega^c$ is conformal. The conformal welding of $\Omega$ is $\Phi_1^{-1}\circ G|_{\partial\mathbb{D}}$. As $\Phi_1$ is conformal and $G$ is harmonic, $\Phi_1^{-1}\circ G$ is the harmonic extension of the welding map. Thus, the HBS of $\Omega$ is: $\mu_{\Phi_1^{-1}\circ G}=\mu_G = B$.

Now, $\Omega$ is determined up to a Mobi\"us transformation $M=\frac{az+b}{cz+d}$. If $G(\infty) = \infty$, $M$ is in the form: $M = az+b= re^{i\theta}z +b$, $r\in \mathbb{R}^+$, $\theta \in [0,2\pi)$ and $b\in \mathbb{C}$. Hence, $\Omega$ is uniquely determined up to a scaling, rotation and translation, which are reflected by $r, \theta$ and $b$ respectively. 

Suppose $B_1, B_2 \in [B]$ are two different HBS, we want to demonstrate that their reconstructed domains $\Omega_1$ and $\Omega_2$ are the same up to scaling, rotation and translation. By definition, their corresponding conformal welding $f_1, f_2$ satisfy $f_2 = M_1^{-1} \circ f_1 \circ M_2$. Let $f_1 = \Phi_1^{-1} \circ \Phi_2$, then $f_2 = (\Phi_1 \circ M_1)^{-1} \circ (\Phi_2 \circ M_2)$ and $H_1, H_2$ are the harmonic extension of $f_1$ and $f_2$ respectively. It's easy to check that the Beltrami coefficient of 
\begin{equation} 
    G_1 := \begin{cases}
    \Phi_1 \circ H_1 \text{ on } \mathbb{\D}\\
    \Phi_2 \text{ on } \mathbb{D}^c
    \end{cases}
\end{equation}
is just $\mu$ in equation (\ref{reconstructionmu}), so $\Omega_1 = T_1 \circ G_1(\D) = T_1 \circ \Phi_1 \circ H_1 (\D)$, where $T_1$ is a composition of scaling, rotation and translation. Note that $H_1$ maps unit disk to unit disk, we have $\Omega_1 = T_1 \circ \Phi_1 (\D)$.
Similarly, let
\begin{equation}
    G_2 := \begin{cases}
    \Phi_1 \circ M_1 \circ H_2 \text{ on } \mathbb{\D}\\
    \Phi_2 \circ M_2 \text{ on } \mathbb{D}^c,
    \end{cases}
\end{equation}
the reconstructed domain of $B_2$ is $\Omega_2 = T_2 \circ G_2(\D) = T_2 \circ \Phi_1 \circ M_1 \circ H_2(\D)$, where $T_2$ is also composed of scaling, rotation and translation. Since $M_1(\D) = H_2(\D) = \D$, $\Omega_2$ can be represented as $\Omega_2 = T_2 \circ \Phi_1(\D)$. Therefore, $\Omega_1= T_1 \circ T_2^{-1} (\Omega_2)$, which shows that the reconstructed shape is invariant (up to a scaling, rotation and translation) regardless of the selection of HBS in given $[B]$.
\end{proof}

The above theorem demonstrates that the HBS is indeed an effective geometric representation or ``fingerprint" of a shape. It determines a shape up to a scaling, rotation and translation. 

\bigskip

\noindent {\it Remark:}
{\it The proof of the above theorem also provides us with a method to reconstruct the shape associated to a given HBS $B$. More precisely, given $B$, we can define a Beltrami coefficient $\mu$ according to equation (\ref{reconstructionmu}). By solving the Beltrami's equation with Beltrami coefficient $\mu$, we obtain the quasiconformal map $G$. If we fix $G$ at $\infty$, the associated shape $\Omega = G(\mathbb{D})$ is uniquely determined up to a rotation, translation and scaling. The associated quasiconformal map can be solved by some computational methods developed earlier \cite{lui2013texture,lui2012optimization,choi2020shape}}

    \begin{figure}
        \begin{center}
            \includegraphics[width=13cm]{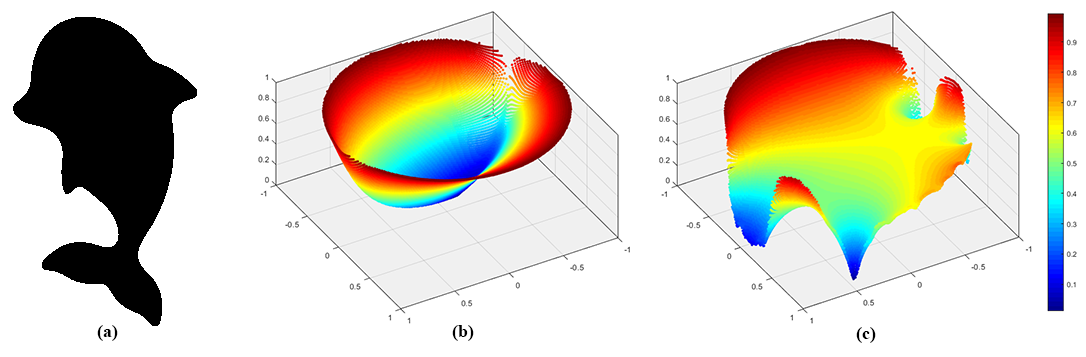}
        \end{center}
        \caption{Illustration of Harmonic Beltrami signature. (a) The input shape, a dolphin; (b) The corresponding harmonic extension, where the conformal welding is shown in figure \ref{fig1} (b); (c) The Harmonic Beltrami signature of (a). Remark that the harmonic function and Harmonic Beltrami signature should be complex-valued function and we only show modulus of them in z-axis in (b) and (c).}
        \label{illu of BS}
    \end{figure}

\subsection{Unique representative of $[B]$}
    As discussed, every shape can be represented by its associated equivalence class of HBS. In order to measure the geometric difference between shapes based on HBS, it is necessary to find a unique representative in the equivalence class $[B]$. Once the unique representatives of two shapes are determined, the geometric difference between them can be easily measured, such as the $L^2$ distance.
            
    In order to proceed to investigate the relationship between $B_1$ and $B_2$, the following theorem is needed.

    \begin{theorem}\label{rotation keep f}
        Suppose $f$ and $\tilde{f}$ are continuous map from $\mathbb{S}^1$ to itself and $\tilde{f} = M_1^{-1} \circ f \circ M_2$, where $M_1, M_2$ are both Mobious transformations. $H$ and $\tilde{H}$ are harmonic extension of $f$ and $\tilde{f}$, then $\tilde{H} = M_1^{-1} \circ H \circ M_2$ iff $M_1$ is a rotation.
    \end{theorem}

    \begin{proof}
        Since $M_1$ is a Mobi\"us transformations, $M_1^{-1}$ is also a Mobi\"us transformation, so it can be written as $M_1^{-1}(z) = e^{i \theta}\frac{z - p}{1 - \overline{p}z}$, where $\theta \in [0, 2\pi)$ and $p \in \D$.

        $\Rightarrow :$ From theorem \ref{composition of harmonic}, we know that $H \circ M_2$ is harmonic since there is no doubt that $M_2$ is conformal. Since $\tilde{H} = M_1^{-1} \circ H \circ M_2$ and $H \circ M_2$ are both harmonic, $M_1^{-1}$ can be represented as $M_1^{-1}(z) = az + b \overline{z} + c$, where $a, b, c \in \C$. Therefore, we have 
        \begin{equation*}
            e^{i \theta}\frac{z - p}{1 - \overline{p}z} = az + b \overline{z} + c,
        \end{equation*}
        which means $p = b = c = 0$, $a = e^{i \theta}$ and so $M_1$ is a rotation.

        $\Leftarrow :$ When $M_1$ is a rotation, $M_1^{-1}$ is also a rotation, so $M_1^{-1} \circ H \circ M_2$ is a harmonic function according to theorem \ref{composition of harmonic}. It's easy to check that
        \begin{equation*}
            M_1^{-1} \circ H \circ M_2 (e^{i \theta})
            = M_1^{-1} \circ f \circ M_2(e^{i \theta})
            = \tilde{f}(e^{i \theta}) = \tilde{H} (e^{i \theta}),
        \end{equation*}
        which means $M_1^{-1} \circ H \circ M_2$ and $\tilde{H}$ have the same boundary value. From the uniqueness of harmonic mapping, $M_1^{-1} \circ H \circ M_2 = \tilde{H}$.
    \end{proof}
            
    Note that if $M_1$ is not only a rotation, $\tilde{f}$ and $\tilde{H}$ also exist but $\tilde{H} \neq M_1^{-1} \circ H \circ M_2$, as shown in figure \ref{harmonic extension not rotation}.

    \begin{figure}
        \begin{center}
            \includegraphics[width=11cm]{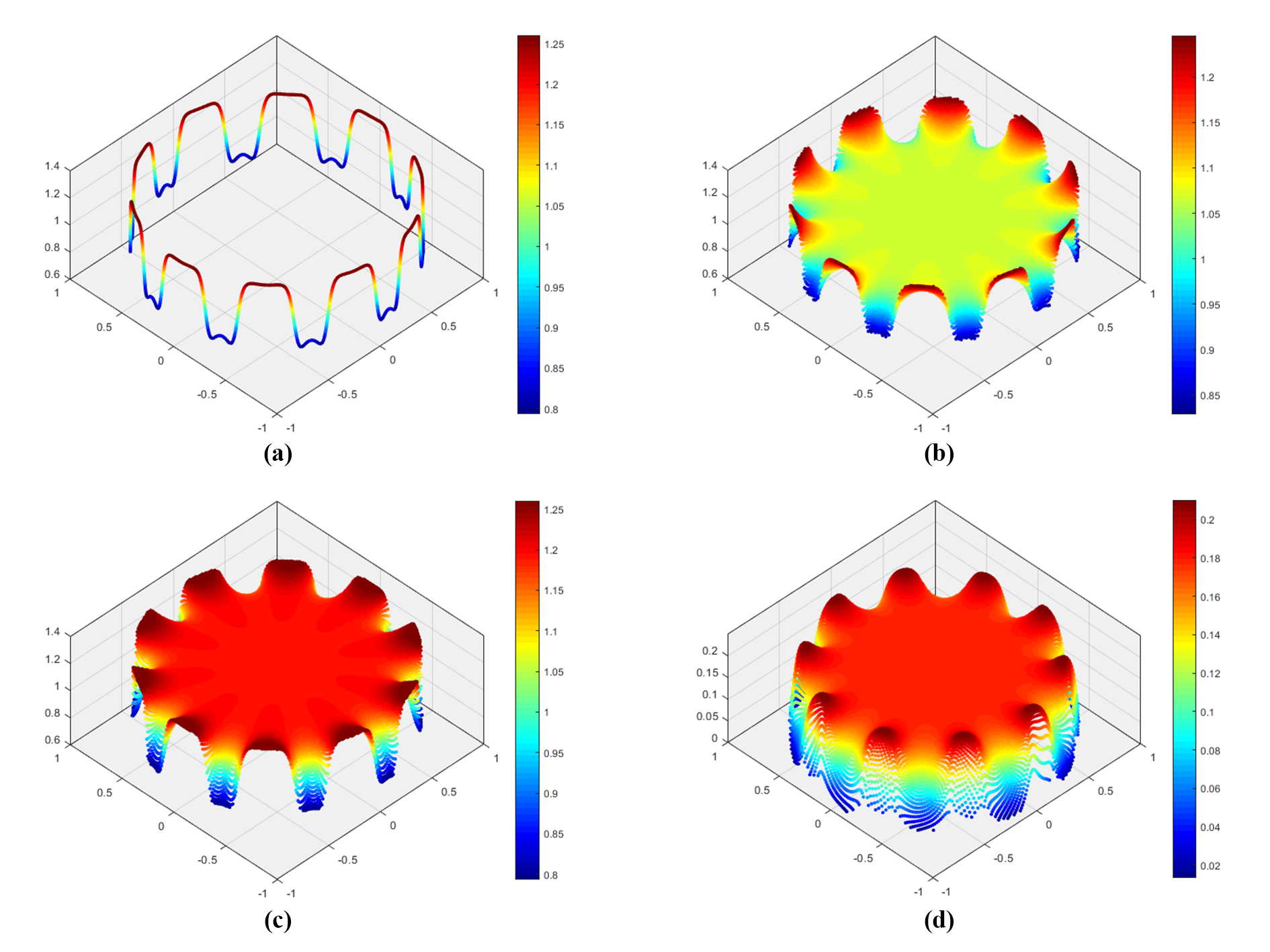}
        \end{center}
        \caption{Let $f(e^{i \theta}) = \sin(10 \theta) + \cos(10 \theta)+1.5$, $M_1(z) = e^{i \tau}\cfrac{z-p}{1-\overline{p}z}$, where $p=0.6+0.6i$, $\tau = 0.8$. The image of $f$ and its harmonic extension $H$ are shown in figure \ref{harmonic}. (a) $\tilde{f} = M_1 \circ f$; (b) harmonic extension $\tilde{H}$ of $\tilde{f}$; (c) $H' = M_1 \circ H$; (d) $H' - \tilde{H}$, and it's clear that $M_1 \circ H \neq \tilde{H}$ except on the boundary.}
        \label{harmonic extension not rotation}
    \end{figure}

    Let $B_1$ and $B_2$ be two harmonic Beltrami signatures in the same equivalence class $[B]$ and $B_k$ is computed by conformal parameterizations $\Phi_1^k:\mathbb{D}\to \Omega$ and $\Phi_2^k:\mathbb{D}^c\to \Omega^c$, where $k=1,2$. Now, we have the following theorem about the relationship between $B_1$ and $B_2$.
            
    \begin{theorem}\label{b1 b2 relationship}
        Suppose $B_k$ and $\Phi_j^k$ are defined as above for $j = 1, 2$ and $k=1,2$. If $\Phi_j^2 = \Phi_j^1\circ M_j$ where $M_j(z) = e^{i\theta_j}z$, then:
        \begin{equation}
          B_2(z) = e^{-2i\theta_2} B_1(e^{i\theta_2}z)  
        \end{equation}
    \end{theorem}

    \begin{proof}
        Under the assumptions, we have: 
        \[
        f_2 = M_1^{-1}\circ f_1 \circ M_2,
        \]
        \noindent where $f_1 = {\Phi_1^1}^{-1}\circ \Phi_2^1$ and $f_2 = {\Phi_1^2}^{-1}\circ \Phi_2^2$. Let $H_1$ and $H_2$ be the harmonic extensions of $f_1$ and $f_2$ respectively. According to theorem \ref{rotation keep f}, we have:
        \[
        H_2 = M_1^{-1}\circ H_1 \circ M_2.
        \]
        Since $M_1$ is a Mobi\"us transformation and thus conformal, we have $B_2 = \mu_{H_2} = \mu_{H_1\circ M_2}$. Using equation (\ref{mu of composition}), we obtain:
        \begin{equation*}
        B_2(z) = \mu_{H_1\circ M_2}(z) = \frac{\bar{M_2}}{M_2} \mu_{H_1}\circ M_2(z) = e^{-2i\theta_2} B_1( e^{i\theta_2} z).
        \end{equation*}
    \end{proof}
    
    When $\Phi_j^1$ and $\Phi_j^2$ are unique up to a rotation for $j=1, 2$, the relationship between $B_1$ and $B_2$ is shown in the above theorem. And with further normalization on the HBS, we can obtain a unique representative of $[B]$.

    \begin{theorem}\label{unique B}
        Suppose $B = \mu_H$ and $\tilde{B} = \mu_{\tilde{H}}$ are two Harmonic Beltrami signatures in given equivalence class $[B]$ for some domain $\Omega$, where $H$ and $\tilde{H}$ are the corresponding harmonic extensions of conformal welding $f$ and $\tilde{f}$ respectively with $\tilde{f} = M_1^{-1} \circ f \circ M_2$. If $M_1$ and $M_2$ are both rotation and
        \begin{gather}
                \arg \int_\D B(z) dz = \arg \int_\D \tilde{B}(z)dz = 0,\label{arg integral B is 0}\\
                \arg \int_\D \frac{B(z)}{z} dz = \arg \int_\D \frac{\tilde{B}(z)}{z}dz \in [0, \pi),\label{arg in 0 and pi}
        \end{gather}
        then we have $\tilde{B} = B$.
    \end{theorem}

    \begin{proof}
        According to theorem \ref{b1 b2 relationship}, the $\tilde{B}$ can be displayed in the form of
        \begin{equation*}
        \tilde{B}(z) = e^{-2 i \theta} B(e^{i\theta} z),
        \end{equation*}
        when $M_1, M_2$ are rotations and $M_2(z) = e^{i \theta} z$. Suppose $\int_\D B(z) dz = r e^{i \tau}$, the integral of $\tilde{B}$ on $\D$ can be written as
        \begin{equation}
            \int_\D \tilde{B}(z) dz
            = \int_\D e^{-2i\theta} B(e^{i\theta} z) dz
            = e^{-2i\theta} \int_\D B(z) dz
            = r e^{i(\tau-2\theta)}.
        \end{equation}
        If equation (\ref{arg integral B is 0}) holds, we have
        \begin{equation*}
            \tau = \tau-2\theta  = 2k\pi ,k \in \Z,
        \end{equation*}
        hence $\theta = k \pi$ and so $\tilde{B}(z) = B(z)$ or $\tilde{B}(z) = B(-z)$. Suppose $\tilde{B}(z) = B(-z)$, then
        \begin{equation*}
            \arg  \int_\D \frac{\tilde{B}(z)}{z} dz 
            = \arg \left(- \int_\D \frac{B(z)}{z} dz \right)
            = \arg \int_\D \frac{B(z)}{z} dz + \pi,
        \end{equation*}
        which can not satisfy (\ref{arg in 0 and pi}), so we must have $\tilde{B} = B$.
    \end{proof}
    
    \bigskip
    
    \noindent {\it Remark:}
    {\it Note that the unique representative of $[B]$ can be easily generated from any HBS $B_0 = \mu_{H_0} \in [B]$ when $M_1$ and $M_2$ are rotations. Suppose $\arg \int_\D B_0(z) dz = \tau_0$, $\tau_0 \neq 0$ and $\arg \int_\D B_0(z)/z dz = \tau_1$, then
    \begin{equation}\label{normaled B}
        B(z) = \begin{cases}
            e^{-i\tau_0}B_0(e^{i\frac{\tau_0}{2}}z), \text{ if } \tau_1 - \frac{\tau_0}{2} \in [0, \pi),\\
            e^{-i\tau_0}B_0(-e^{i\frac{\tau_0}{2}}z), \text{ if } \tau_1 - \frac{\tau_0}{2} \in [\pi, 2\pi),
            \end{cases}
    \end{equation}
    is just the desired representative since
    \begin{gather*}
        \arg \int_\D B(z) dz = \arg \left( e^{-i\tau_0} \int_\D B_0(\pm e^{i\frac{\tau_0}{2}}z)dz \right) = \tau_0-\tau_0 = 0,
        % \arg \int_\D \frac{B(z)}{z} dz = 
        % \begin{cases}
        % \arg \left( e^{-i\frac{\tau_0}{2}} \int_\D \frac{B_0(e^{i\frac{\tau_0}{2}}z)}{e^{i\frac{\tau_0}{2}}z}dz \right) = \tau_1 - \frac{\tau_0}{2} \in [0, \pi), \text{ if } \tau_1 - \frac{\tau_0}{2} \in [0, \pi),\\
        % \arg \left(-e^{-i\frac{\tau_0}{2}} \int_\D \frac{B_0(-e^{i\frac{\tau_0}{2}}z)}{-e^{i\frac{\tau_0}{2}}z}dz \right) = \tau_1 - \frac{\tau_0}{2}  - \pi \in [0, \pi), \text{ if } \tau_1 - \frac{\tau_0}{2} \in [\pi, 2\pi).
        % \end{cases}
    \end{gather*}
    for $\tau_1 - \frac{\tau_0}{2} \in [0, \pi)$
    \begin{gather*}
        \arg \int_\D \frac{B(z)}{z} dz = \arg \left( e^{-i\frac{\tau_0}{2}} \int_\D \frac{B_0(e^{i\frac{\tau_0}{2}}z)}{e^{i\frac{\tau_0}{2}}z}dz \right) = \tau_1 - \frac{\tau_0}{2} \in [0, \pi),
    \end{gather*}
    and for $\tau_1 - \frac{\tau_0}{2} \in [\pi, 2\pi)$
    \begin{gather*}
        \arg \int_\D \frac{B(z)}{z} dz = \arg \left(-e^{-i\frac{\tau_0}{2}} \int_\D \frac{B_0(-e^{i\frac{\tau_0}{2}}z)}{-e^{i\frac{\tau_0}{2}}z}dz \right) = \tau_1 - \frac{\tau_0}{2}  - \pi \in [0, \pi).
    \end{gather*}
    Therefore, the only difficulty on the way to the unique representative $B$ of $[B]$ is how to normalize $M_1$ and $M_2$ to be rotations.
    }

    \subsection{Normalization to $M_1$}\label{norm_to_m1}
    Suppose $\Phi_1:\mathbb{D}\to \Omega$ be the conformal parameterization of $\Omega$. We proceed to normalize $\Phi_1$ so that $M_1$ is a rotation. One possible way is to fix $\Phi_1$ on some special points, like $\Phi_1(0) = 0$. But it requires the assumption that these points are inside $\Omega$, which is equivalent to limiting the position of shape. In this work, we introduce a new approach without any additional assumption.
    
    In practical application, we usually use finite boundary points $z_1, z_2, \cdots, z_n \in \partial \Omega$ to represent $\Omega$. Denote $p_k = \Phi_1^{-1}(z_k) \in \partial \D$ for $k=1,\cdots,n$, we claims that $M_1$ can be normalized by restricting the arithmetic mean of $p_k$ to be $0$, that is
    \begin{equation}\label{norm phi1}
        \sum^n_{k=1} \Phi_1^{-1}(z_k) = \sum_{k=1}^n p_k = 0.
    \end{equation}
    
    A natural question is whether there exists a conformal parameterization $\Phi_1$ satisfying equation (\ref{norm phi1}). Actually, the existence is equivalent to that given boundary points $p_k$ on unit circle, there is a Mobi\"us transformation $M$ such that $\sum_{k=1}^n M(p_k) = 0$. Without lose of generality, we ignore the rotational component of $M$ and let $M = F_a(z) = \frac{z - a}{1 - \overline{a}z}$, where $a \in \D$. We proceed to solve
    \begin{equation}\label{eq}
        f(a) = \sum_{k=1}^n F_a(p_k) = \sum_{k=1}^n \frac{p_k - a}{1 - \overline{a}p_k} = 0.
    \end{equation}

    \begin{figure}
        \begin{center}
            \includegraphics[width=10cm]{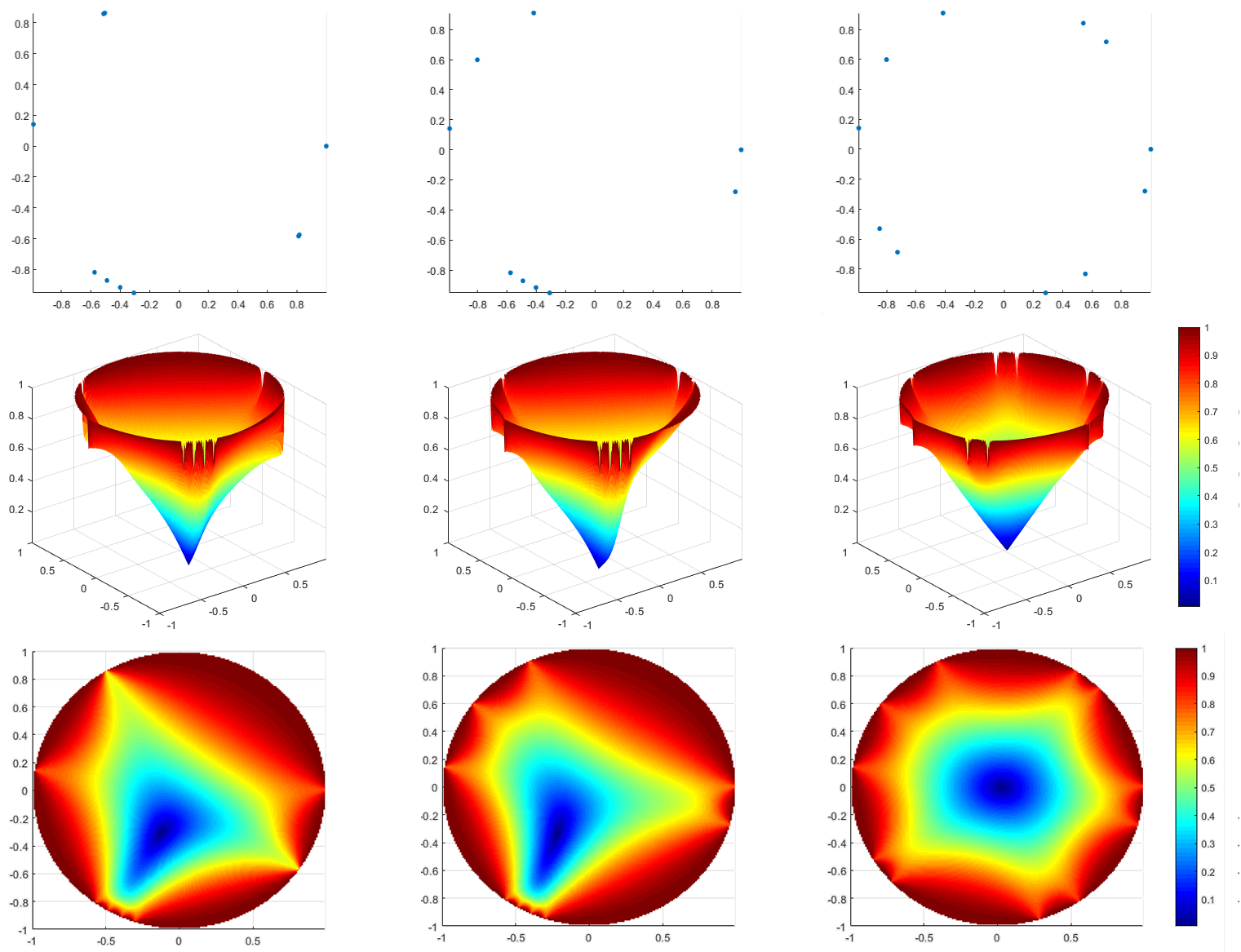}
        \end{center}
        \caption{The first row are some randomly generated points in $\partial \D$. The second row are the corresponding $\abs{f(a)}$, where $a \in \D$. The third row are also $\abs{f(a)}$ but is in top view. }
        \label{fasolvable}
    \end{figure}

    Figure \ref{fasolvable} shows some graphs of $|f(a)|$ corresponding to various boundary points and illustrates equation (\ref{eq}) may be solvable. And with the help of Brouwer fixed point theorem, the existence of the solution to equation (\ref{eq}) can be achieved.

    \begin{theorem}\label{existence}
        Given $\{p_1, p_2, \cdots, p_n\} \subset \partial \D$ and $n \ge 3$, let $F_a(z) = \frac{z - a}{1 - \overline{a}z}$, where $a \in \D$. The solution of equation (\ref{eq}) always exists.
    \end{theorem}

    \begin{proof}
        Note that when $e^{i \theta} \neq p_k$ for any $k$, we have
        \begin{equation*}
            f(e^{i \theta}) = \sum_{k=1}^n \frac{p_k - e^{i \theta}}{1 - e^{-i\theta}p_k} = - n e^{i\theta},
        \end{equation*}
        so $\frac{1}{n} f(e^{i\theta}) + e^{i \theta} = 0$. Let 
        \begin{equation*}
            g(a) = \begin{cases}
        \frac{1}{n} f(a) + a, a \in \D,\\
        0, a \in \partial \D,
        \end{cases}
        \end{equation*}
        such $g$ is continuous on $\overline{\D}$. It's clear that $\norm{g(a)} \le \frac{1}{n} \norm{f(a)} + \norm{a} \le 2$.
        
        Let $M = \{a \in \D ~|~ \norm{g(a)} \ge 1 \}$, we define
        \begin{equation*}
            h(a) = \begin{cases}
            g(a), &a \in  \overline{\D} \setminus M,\\
            \frac{g(a)}{\norm{g(a)}}, &a \in  M.
        \end{cases}
        \end{equation*}
        Since $h$ is a continuous map from $\overline{\D}$ to itself, there exists some $a \in \overline{\D}$ such that $h(a) = a$.
        
        If $a \in \partial \D$, $\norm{a} = 1$ but $\norm{h(a)} = 0 \neq \norm{a}$. If $a \in M$, we know $a \notin \partial \D$, so $\norm{a} < 1$, but $\norm{h(a)} = \frac{\norm{g(a)}}{\norm{g(a)}} = 1 \neq \norm{a}$. So there exists $a \in \D \setminus M$ such that 
        \begin{equation*}
            h(a) = g(a) = \frac{1}{n}f(a) + a = a,
        \end{equation*}
        which means $f(a) = 0$.
    \end{proof}

    The uniqueness of the solution of equation (\ref{eq}) is another thing we concern. Firstly, we consider a special case.
    
    \begin{theorem}\label{uniqueness when 0}
        Suppose $\sum_{k=1}^n p_k = 0$, equation (\ref{eq}) holds if and only if $a = 0$.
    \end{theorem}

    \begin{proof}
        If $a = 0$, it's obvious $F_0$ is identity, so $f(0) = \sum_{k=1}^n p_k = 0$.
        
        If $a \neq 0$, WLOG, we can assume that $\arg(a) = 0$ and $0 < a < 1$. We have $F_a(1) = 1$ and $F_a(-1) = -1$. As for $p_k \neq \pm 1$, there is some $\theta_k \in (0, \pi) \cup (\pi, 2\pi)$ such that $p_k = \cos \theta_k + \sin \theta_k i$ and
        \begin{align*}
            &F_a(p_k) \\
            = &F_a(\cos \theta_k + \sin \theta_k i) \\
            = &\frac{\cos \theta_k + \sin \theta_k i - a}{1 - a \cos \theta_k - a \sin \theta_k i} \\
            = &\frac{(a^2 + 1) \cos \theta_k -2a - (a^2 - 1) \sin \theta_k i}{a^2 + 1 - 2a \cos \theta_k},
        \end{align*}
        so
        \begin{align*}
            &\Re(F_a(p_k)) \\
            = &\frac{(a^2 + 1) \cos \theta_k -2a}{a^2 + 1 - 2a \cos \theta_k} \\
            = &\cos \theta_k - \frac{2a(1-\cos^2 \theta_k)}{(a-1)^2 + 2a(1-\cos \theta_k)} \\
            < &\cos \theta_k = \Re(p_k).
        \end{align*}

        Therefore, if $n \ge 3$, there must be at least one $p_k \neq \pm 1$ and so
        \begin{equation*}
            \Re(\sum_{k=1}^n F_a(p_k)) = \sum_{k=1}^n \Re(F_a(p_k)) < \sum_{k=1}^n \Re(p_k) = \Re(\sum_{k=1}^n p_k)  = 0,
        \end{equation*}
        which means $\sum_{k=1}^n F_a(p_k) \neq 0$. Note that we can rotate all $p_k$ to get the same conclusion when $\arg(a) \neq 0$. So $a=0$ is the only solution of equation (\ref{eq}) when $\sum_{k=1}^n p_k = 0$.
    \end{proof}

    This above theorem confirms the uniqueness for a special situation but actually this conclusion is universal no matter how $p_i$ distribute.

    \begin{theorem}\label{uniqueness}
        The solution of equation (\ref{eq}) is unique.
    \end{theorem}

    \begin{proof}
        Assume that $a_0, a_1$ are two solutions, then we have that
        \begin{align*}
            \sum_{k=1}^n F_{a_0}(p_k) = 0,\sum_{k=1}^n F_{a_1}(p_k) = 0
        \end{align*}
        
        Let $p_k' = F_{a_0}(p_k)$, then
        \begin{equation*}
            \sum_{k=1}^n F_{a_1} (p_k)
            = \sum_{k=1}^n (F_{a_1} \circ F_{a_0}^{-1})(F_{a_0}(p_k))
            = \frac{1- a_1 \overline{a_0}}{1- a_0 \overline{a_1}}\sum_{k=1}^n F_{\frac{a_1-a_0}{1-a_1\overline{a_0}}}(p_k') = 0.
        \end{equation*}
        Since $a_0, a_1 \in \D$, then $\frac{1- a_1 \overline{a_0}}{1- a_0 \overline{a_1}} \neq 0$ and so
        $$\sum_{k=1}^n F_{\frac{a_1-a_0}{1-a_1\overline{a_0}}}(p_k') = 0.$$
        According to theorem \ref{uniqueness when 0}, $\frac{a_1-a_0}{1-a_1\overline{a_0}} = 0$, then $a_0 = a_1$.
    \end{proof}

    With the above observations, we can come back to the original problem about normalization $M_1$ and have the following theorem.

    \begin{theorem}\label{unique up to a rotation}
        Given $\{z_1, z_2, \cdots, z_n\} \subset \partial \Omega$ and $n \ge 3$, if conformal mapping $\Phi_1: \D \rightarrow \Omega$ satisfies equation (\ref{norm phi1}), then such $\Phi_1$ is unique up to a rotation $M_1$.
    \end{theorem}

    \begin{proof}
        Suppose $\Phi_1$ and $\tilde{\Phi}_1$ are two arbitrary conformal map from $\D$ to $\Omega$  then $\tilde{\Phi}_1 = \Phi_1 \circ M_1$, where $M_1$ is a Mobi\"us transformation. If $\Phi_1$, $\tilde{\Phi}_1 $ satisfy equation (\ref{norm phi1}), let $p_k = \Phi_1^{-1}(z_k)$, then we have $\sum_{k=1}^n p_k = 0$ and $\sum_{k=1}^n M_1^{-1} (p_k) = 0$. $M_1^{-1}$ is also a Mobi\"us transformation so let $M_1^{-1}(z) = e^{i \theta}\frac{z - a}{1 - \overline{a} z} = e^{i \theta} F_a(z)$, then we have
        \begin{equation*}
            e^{i \theta} \sum_{k=1}^n F_a(p_k) = 0.
        \end{equation*}

        From theorem \ref{uniqueness when 0} we can know $a = 0$, then $M_1^{-1}(z) = e^{i \theta} z$ and $\Phi_1$ and $\tilde{\Phi}_1$ are the same up to a rotation $M_1$.
    \end{proof}

    \subsection{Normalization to $M_2$}\label{norm2}
    As for $M_2$, we also hope it is a rotation, which is equivalent to that $\Phi_2$ is uniquely determined up to a rotation. Luckily, $\infty$ is always inside $\D^c$ and $\Omega^c$, and we can use this to ensure $M_2$ to be a rotation.

    \begin{theorem}
        Let $\Phi_2$ be a conformal map from $\D^c$ to $\Omega^c$ satisfying
        \begin{equation}
            \Phi_2(\infty) = \infty, \label{norm phi2 1}
        \end{equation}
        then such $\Phi_2$ is uniquely determined up to a rotation.
    \end{theorem}

    \begin{proof}
        Let $\Phi_2$ and $\tilde{\Phi}_2$ be two arbitrary conformal map from $\D^c$ to $\Omega^c$, then $\tilde{\Phi}_2 = \Phi_2 \circ M_2$, where $M_2$ is Mobi\"us transformation and $M_2(z) = e^{i \theta} \frac{z - a}{1- \overline{a}z}$. Since $\Phi_2$ and $\tilde{\Phi}_2$ both satisfy equation (\ref{norm phi2 1}), then
        \begin{equation*}
            \tilde{\Phi}_2(\infty) = \Phi_2(M_2(\infty)) = \infty, \Phi_2(\infty) = \infty.
        \end{equation*}
        Therefore, $M_2$ maps $\infty$ to $\infty$, which means that $a = 0$ and $M_2(z) = e^{i \theta} z$.
    \end{proof}

\subsection{Invariance under simple transformation}
    With the normalization mentioned above, we can get a unique HBS $B$ as the representative corresponding to domain $\Omega$, so we can remark $B$ as $B_\Omega$. Now we want to prove that if we do some simple transformation like rotation, scaling and translation to $\Omega$, the HBS is invariant.

    \begin{theorem}\label{invariance theorem}
        Given a boundary simply-connected domain $\Omega$ and transformation $T$ composed of rotation, scaling and transformation. Let $B_\Omega$ and $B_{T(\Omega)}$ be the HBS of $\Omega$ and $T(\Omega)$, then $B_\Omega = B_{T(\Omega)}$.
    \end{theorem}

    \begin{proof}
        Suppose $\Phi_1 : \D \rightarrow \Omega$, $\Phi_2: \D^c \rightarrow \Omega^c$, $\tilde{\Phi}_1: \D \rightarrow T(\Omega)$ and $\tilde{\Phi}_2 : \D^c \rightarrow T(\Omega)$ are conformal. Since $T$ is composed of rotation, scaling and translation, $T$ can be written as $T(z) = ke^{i\theta} z + b$ and such $T$ is absolutely invertible and conformal.
        
        Let $\hat{\Phi}_1 =  T^{-1} \circ \tilde{\Phi}_1 : \D \rightarrow \Omega$, $\hat{\Phi}_1$ is conformal. Given the boundary points $\{z_1, z_2, \cdots, z_n\} \subset \partial \Omega$, then $\{T(z_1), T(z_2), \cdots, T(z_n)\} \subset \partial T(\Omega)$. Since $\tilde{\Phi}_1$ satisfies condition (\ref{norm phi1}), we have
        \begin{equation*}
            \sum_{i=1}^n \hat{\Phi}_1^{-1}(z_i) = \sum_{i=1}^n \tilde{\Phi}_1^{-1} \circ T(z_i) = \sum_{i=1}^n \tilde{\Phi}_1^{-1}(T(z_i)) = 0,
        \end{equation*}
        which means $\hat{\Phi}_1$ also satisfies condition (\ref{norm phi1}). Hence $\hat{\Phi}_1 =  T^{-1} \circ \tilde{\Phi}_1= \Phi_1 \circ M_1$, which equals to
        \begin{equation}
            \tilde{\Phi}_1 = T \circ \Phi_1 \circ M_1,
        \end{equation}
        where $M_1$ is a rotation.

        Similarly, let $\hat{\Phi}_2 = T^{-1} \circ \tilde{\Phi}_2: \D^c \rightarrow \Omega^c$, $\hat{\Phi}_2$ is conformal. Since $\tilde{\Phi}_2(\infty) = \infty$, we have
        \begin{equation*}
            \hat{\Phi}_2(\infty) = T \circ \tilde{\Phi}_2(\infty) = \infty,
        \end{equation*}
        which means $\hat{\Phi}_2$ satisfies condition (\ref{norm phi2 1}). Therefore, $\hat{\Phi}_2 = T^{-1} \circ \tilde{\Phi}_2 = \Phi_2 \circ M_2$ and then
        \begin{equation}
            \tilde{\Phi}_2 = T \circ \Phi_2 \circ M_2,
        \end{equation}
        where $M_2$ is also a rotation.

        The corresponding conformal welding of $T(\Omega)$ is
        \begin{equation*}
            \tilde{f}
            = \tilde{\Phi}_1^{-1} \circ \tilde{\Phi}_2 
            = M_1^{-1} \circ \Phi_1^{-1} \circ T^{-1} \circ T \circ \Phi_2 \circ M_2 
            = M_1^{-1} \circ f \circ M_2,
        \end{equation*}
        which means $B_\Omega$ and $B_{T(\Omega)}$ are both the representative of the same equivalence class. Therefore, $B_\Omega = B_{T(\Omega)}$ because of theorem \ref{unique B}.
    \end{proof}

\subsection{Geometric implication of HBS}\label{geometric implication}
    Although there exists a one-to-one correspondence between HBS and shapes up to a rotation, translation and scaling, the geometric implication of HBS is still unknown. More precisely, if two shapes are compared based on the HBS, it is necessary to know whether two shapes are close when their corresponding HBS are close. To study this, we first define the distance between two shapes $\Omega_1\subset \C$ and $\Omega_2\subset \C$ as follows: 
    \begin{equation}\label{shape distance}
        d_\Omega(\Omega_1,\Omega_2) = \frac{1}{2}\left(\max_{q\in\partial \Omega_2} \min_{p\in\partial \Omega_1} ||p-q|| + \max_{p\in\partial \Omega_1} \min_{q\in\partial \Omega_2} ||p-q||\right),
    \end{equation}
    \noindent where $||\cdot||$ refers to the Euclidean norm. We shall show that $d_\Omega(\Omega_1,\Omega_2)$ is small if their Harmonic Beltrami Signatures are alike.
    
The following theorem is useful, which describe the perturbation of the quasiconformal map under the perturbation of the associated Beltrami coefficient.

    \begin{theorem}[Beltrami holomorphic flow on $\overline{\C}$]\label{BHF}
        There is a one-to-one correspondence between the set of quasiconformal diffeomorphisms of $\overline{\C}$ that fix the points 0, 1, and $\infty$ and the set of smooth complex-valued functions $\mu$ on $\overline{\C}$ with $\norm{\mu}_\infty = k < 1$. Here, we have identified $\overline{\C}$ with the extended complex plane $\overline{\C}$. Furthermore, the solution $f^\mu$ to the Beltrami equation depends holomorphically on $\mu$. Let $\{\mu(t)\}$ be a family of Beltrami coefficients depending on a real or complex parameter $t$. Suppose also that $\mu(t)$ can be written in the form
        \begin{equation}
            \mu(t)(z) = \mu(z) + tv(z) + t \epsilon(t)(z),
        \end{equation}
        \noindent with suitable $\mu$ in the unit ball of $C^\infty(\C)$, $v, \epsilon(t) \in L^\infty(\C)$ such that $\lim_{t \rightarrow 0} \norm{\epsilon(t)}_\infty = 0$. Then for all $w \in \C$,
        \begin{equation}
            f^{\mu(t)}(w) = f^{\mu}(w) + tV(f^\mu, v)(w) + o(\abs{t})
        \end{equation}
        locally uniformly on $\C$ as $t \rightarrow 0$, where
        \begin{align}
            &V(f^\mu, v)(w) = -\frac{f^\mu(w)(f^\mu(w)-1)}{\pi}W(f^\mu, v)(w)\\
            &W(f^\mu, v)(w) = \int_\C \frac{v(z)(f^\mu)^2_z(z)}{f^\mu(z)(f^\mu(z)-1)(f^\mu(z) - f^\mu(w))}dz.
        \end{align}
    \end{theorem}

    \begin{proof}
        This theorem is due to Bojarski. For detailed proof, please refer to \cite{durrleman2007measuring}.
    \end{proof}
    
    Recall that given a HBS, its associated shape is uniquely determined up to a translation, scaling and rotation. Therefore, in order to analyze the geometric implication of HBS, we shall normalize the shape associated to a given HBS. According to theorem \ref{one to one equivalence class}, given $[B] \in \mathcal{B}$, its corresponding shape can be determined by computing a quasiconformal map $G$ associated to $\mu$ given by equation (\ref{reconstructionmu}) and the shape $\Omega = G(\mathbb{D})$ can be reconstructed. In particular, $\Omega$ can be normalized by constraining $G$ to fix $0, 1$ and $\infty$. In this subsection, we assume the shape $\Omega$ corresponding to a HBS is normalized as described above.
    
    Now, the geometric implication of HBS can be explained by the following theorem.
    
    \begin{theorem}\label{reconstruction robust theorem}
    Let $[B_1]$, $[B_2]$ be two equivalence class of HBS and $B_1$, $B_2$ be the unique representatives. Let $\Omega_1$ and $\Omega_2$ be the normalized shapes associated to $B_1$ and $B_2$ respectively. If 
    $||B_1 - B_2||_{\infty} < \epsilon$, then $d_\Omega(\Omega_1,\Omega_2)<\frac{2M}{\pi}\epsilon$ for some $M>0$.
    \end{theorem}
    \begin{proof}
    According to theorem \ref{one to one equivalence class}, $\Omega_1$ and $\Omega_2$ can be reconstructed by solving for the quasiconformal maps $G_1$ and $G_2$. Then, $\Omega_1 = G_1(\D)$ and $\Omega_2 = G_2(\D)$.
    
    Let
    \begin{equation}
        g(t)(z) = \begin{cases}
            B_1(z) + t v(z), z \in \D \\
            0, z \in \D^c
        \end{cases}
    \end{equation}
    where $v(z) = \frac{B_2(z) - B_1(z)}{||B_1 - B_2||_{\infty}}$ if $z \in \D$ and $v(z) = 0$ if $z \notin \D$. Then, $G_1$ and $G_2$ are quasiconformal maps associated to the Beltrami coefficients $g(0)$ and $g(t)$ respectively, where $t = ||B_1 - B_2||_{\infty} \in (0, \epsilon)$. According to theorem \ref{BHF},
    \begin{align*}
        &\norm{G_2 - G_1}_\infty = \norm{t V(G_1, v) + o(t)}_\infty \\
        \le& \frac{t}{\pi}\norm{G_1}_\infty \norm{G_1 - 1}_\infty \norm{W(G_1, v)}_\infty + o(t)\\
        \le& \frac{2t}{\pi} \norm{W(G_1, v)}_\infty + o(t).
    \end{align*}
    Since $G_1$ is continuous and bounded, the following integral is bounded and for any $w \in \D$ there exists some $M > 0$ such that
    \begin{equation}
    \begin{split}
        &\norm{\int_\D \frac{(G_1)_z^2(z)}{G_1(z)(G_1(z)-1)(G_1(z)-G_1(w)}dz} \\
        \le& \int_\D \norm{\frac{(G_1)_z^2(z)}{G_1(z)(G_1(z)-1)(G_1(z)-G_1(w)}} dz \\
        \le& M.
    \end{split}
    \end{equation}
    Therefore, we have $\norm{W(G_1, v)}_\infty \le M \norm{v}_\infty \le M$ and
    \begin{equation}
        \norm{G_2 - G_1}_\infty \le \frac{2t}{\pi} \norm{W(G_1, v)}_\infty \le \frac{2Mt}{\pi} \le \frac{2M}{\pi} \epsilon.
    \end{equation}
    
    Now, for any $q = G_2(z) \in \Omega_2$ ($z\in \D$), we have $\min_{p\in \Omega_1} ||p-q|| \leq ||G_1(z) - G_2(z)||\leq \frac{2M}{\pi}\epsilon$. Thus, $\max_{q\in \Omega_2}\min_{p\in \Omega_1} ||p-q|| \leq \frac{2M}{\pi}\epsilon$. Similarly, $\max_{p\in \Omega_1}\min_{q\in \Omega_2} ||p-q|| \leq \frac{2M}{\pi}\epsilon$. As a result, we have 
    \[
    d_\Omega(\Omega_1,\Omega_2)<\frac{2M}{\pi}\epsilon
    \]
    \end{proof}
    
    This above theorem illustrates that if the difference between HBS is small enough, their corresponding domains is almost the same, which means our Harmonic Beltrami signature is a good similarity indicator of shapes.

\section{Implementation detail}\label{implementation}
\subsection{Zipper algorithm}
    In order to find a unique and stable HBS, the first thing is to find a way to calculate a conformal mapping from the given domain to unit disk. As mentioned in Section \ref{norm_to_m1}, we only have finite boundary points of the shape and zipper algorithm invented in the 1980s is a suitable and accurate method to deal with this situation numerically. 

    Marshall \etal demonstrates the zipper algorithm detailedly with clear diagrams in \cite{marshall2007convergence}. For the convenience of readers, we gives a very brief review here. Given $N$ clockwise boundary points $z_1, z_2, \cdots, z_N \in \partial \Omega$, this algorithm use a series of linear fractional transformations $g_1, g_2, \cdots, g_N$ to map $z_1, z_2,\cdots, z_N$ to real axis one-by-one, and finally transform the upper half plane to unit disk by $g_{N+1}(z) = \frac{z-i}{z+i}$. Therefore, $g = g_{N+1} \circ g_N \circ \cdots \circ g_2 \circ g_1$ is a conformal mapping indisputably and maps all these boundary points to unit circle and the domain $\Omega$ to $\D$. Remark that zipper algorithm is sensitive to the order of points. If we input the points anti-clockwise, i.e. $z_N, z_{N-1}, \cdots, z_1$, the zipper will give us a conformal mapping from $\Omega^c$ to $\D$. This progress is shown in figure \ref{zipper algo}.

    \begin{figure}
    \begin{center}
        \includegraphics[width=12cm]{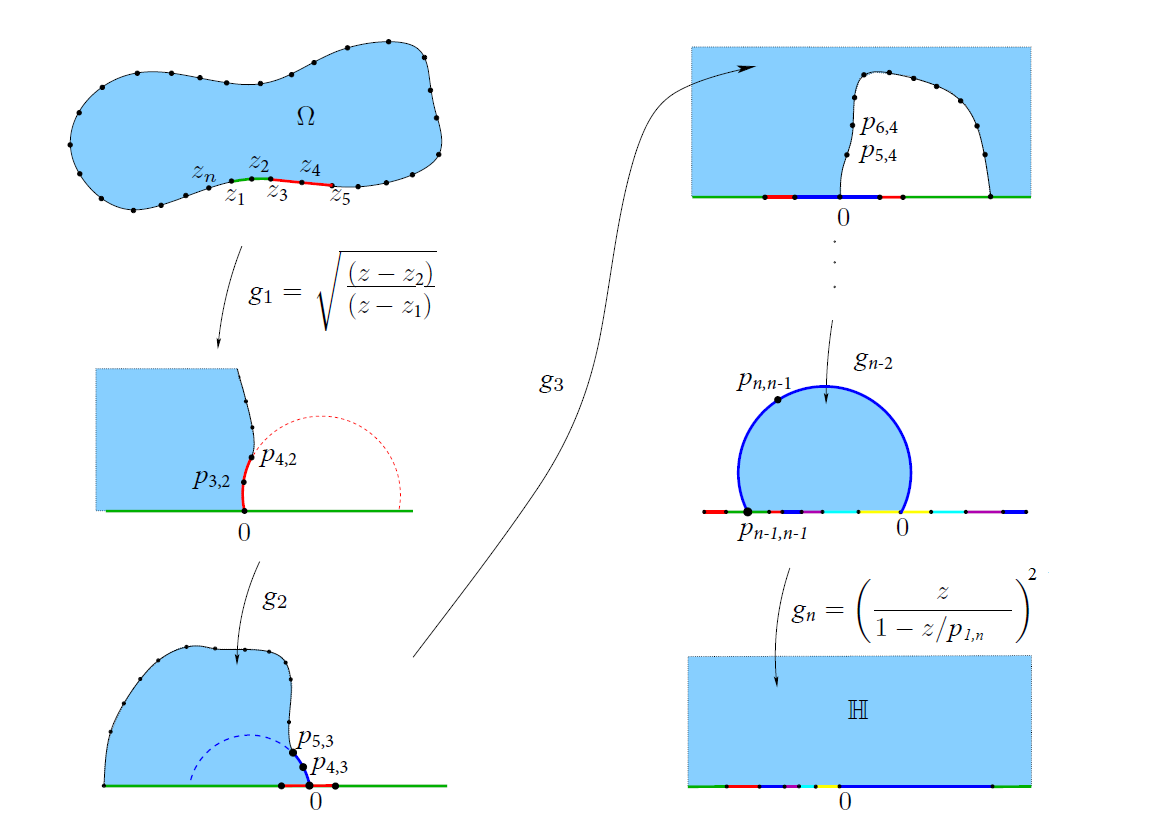}
    \end{center}
    \caption{Zipper algorithm}
    \label{zipper algo}
    \end{figure}

    For $\Phi_1 : \D \rightarrow \Omega$, we can find a conformal mapping $g_{\Phi_1}: \Omega \rightarrow \D$ by inputting points clockwise, and $\Phi_1 = g_{\Phi_1}^{-1}$. For $\Phi_2 : \D^c \rightarrow \Omega^c$, we can input anti-clockwise points and get $g_{\Phi_2} : \Omega^c \rightarrow \D$, then $\Phi_2(z) = g_{\Phi_2}^{-1}(\frac{1}{z})$. Because of the invariance of HBS under scaling, the number of boundary points $N$ can be fixed as $200$ here and these points are picked uniformly from the shape contour. 

        \begin{algorithm}[H]                           % HERE!!!!!!!!!
        \caption{Zipper}          % give the algorithm a caption
        \label{zipper}      % and a label for \ref{} commands later in the document
        \begin{algorithmic}  % enter the algorithmic environment
            \STATE \textbf{Inputs:} $z_i \in \partial \Omega$ for $i=1, 2, \cdots, N$, $N=200$. %, $\epsilon = 10^{-10}$(any very small value)
            \STATE \textbf{Initialize:} Let $k=2$, $g_1(z) = \sqrt{\frac{z-z_2}{z-z_1}}$, $g= g_1$ and compute $p_{i,2}=g(z_i)$.
            \WHILE{$k < N$}
                \STATE Pick $q = p_{k+1, k} = a + b i$, then compute $c = \frac{a}{\abs{q}^2}$, $d = \frac{b}{\abs{q}^2}$.
                \STATE Let $g_k(z) = \sqrt{\frac{cz}{1+dzi}}$, then $g = g_k \circ g$ and compute $p_{i, k+1} = g_k(p_{i, k})$.
                \STATE Let $k = k+1$.
            \ENDWHILE
            \STATE Let $g_{N}(z) = \left(\frac{z}{1-\frac{z}{p_{1,N}}}\right)^2$ and $g_{N+1}(z) = \frac{z-i}{z+i}$, then $g = g_{N+1} \circ g_N \circ g$ and $p_i = g_{N+1} \circ g_N(p_{i,N})$.
            \RETURN Conformal mapping $g: \Omega \rightarrow \D$ and boundary points $p_i \in \partial \D$.
        \end{algorithmic}
        \end{algorithm}

    \subsection{Normalization}
    Section \ref{norm_to_m1} and \ref{norm2} show that we can normalize $M_1$ and $M_2$ by some restrictions and then HBS $B$ can be unique. This section will tells how to satisfy equation (\ref{norm phi1}) and (\ref{norm phi2 1}) from the output of zipper algorithm.

    To normalization $M_1$, we need to solve equation (\ref{eq}). Generally speaking, the output of zipper, $p_i = g_{\Phi_1}(z_i) \in \partial \D$ for $i = 1, 2, \cdots, N$, gather around some point. At that time, $\abs{f(a)} \approx 1$ almost everywhere and the solution of (\ref{eq}) is also very close to that point. This means the solution is quite unstable and hard to converge for common algorithms(see figure \ref{original distribution}).

    \begin{figure}
        \begin{center}
            \includegraphics[width=9cm]{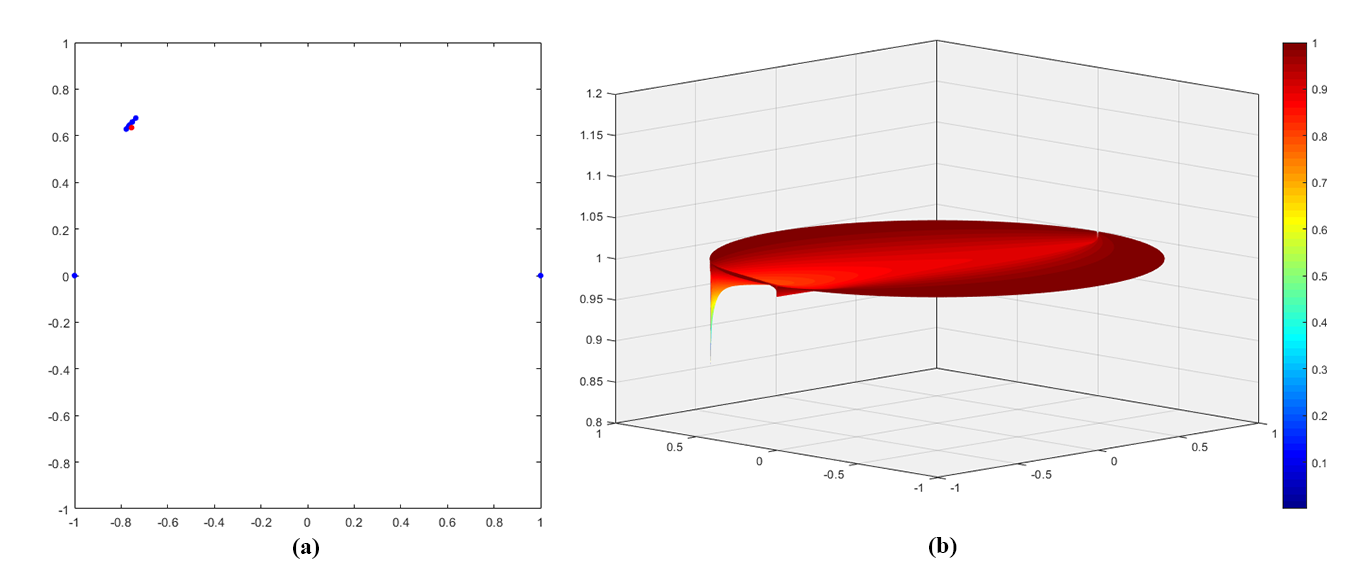}
        \end{center}
        \caption{(a) The $p_i \in \partial \D$ gather in a small neighborhood around their arithmetic mean $p_c$, which is labeled in red; (b) The corresponding $\abs{f(a)}$ for these $p_i$. It's worth to mention that the minimal of $\abs{f(a)}$ can reach actually, but it isn't shown in the picture since the grid is not small enough.}
        \label{original distribution}
    \end{figure}

    \begin{figure}
        \begin{center}
            \includegraphics[width=9cm]{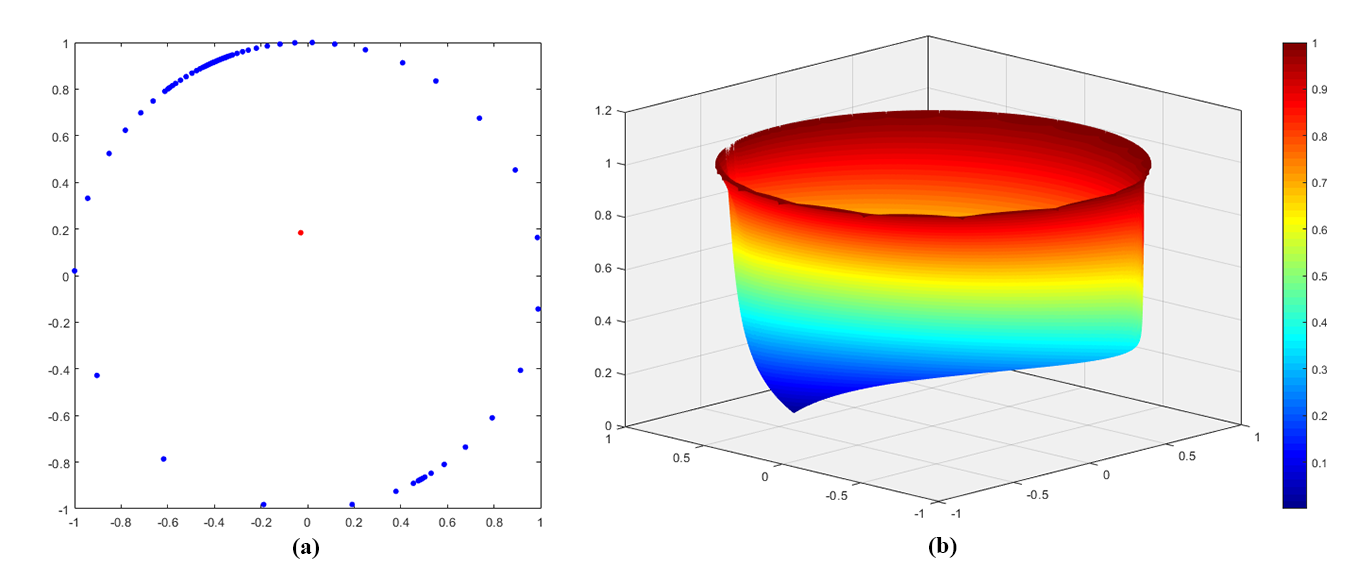}
        \end{center}
        \caption{Similar with figure \ref{original distribution}. (a) The boundary points after adjustment; (b) The $\abs{f(a)}$.}
    \end{figure}

    Instead of proposing a complicated method to solve equation directly, the way we adopted to solve this problem is to use some Mobi\"us transformations to adjust the distribution of $p_{i}$ until their arithmetic center is very close to 0. We set $p_{i,0} = p_{i}$ at the beginning and for the $k$-th iteration, let $p_{i,k} \in \partial \D$ be the boundary points and
    \begin{equation*}
        p_{c,k} = \frac{\sum_{i=1}^N p_{i,k}}{N} \in \D
    \end{equation*}
    as the arithmetic center of $p_{i,k}$. Remark that $F_a(z) = \frac{z-a}{1-\overline{a}z}$ is a Mobi\"us transformation ignoring rotation, then the $F_{p_{c,r}}$ gives new boundary points as 
    \begin{equation*}
        p_{i,k+1} = F_{p_{c,k}}(p_{i,k}) = \frac{p_{i,k} - p_{c,k}}{1 - \overline{p_{c,k}}p_{i,k}}.
    \end{equation*}
    Repeat this iteration for $k$ times until $\abs{p_{c,k+1}}$ is close to $0$, then the distribution of $p_{i, k+1}$ is sufficiently regular and almost uniform. At that time, suppose 
    \begin{equation*}
        M_{\Phi_1} = F_{p_{c,k}} \circ F_{p_{c,k-1}} \circ \cdots \circ F_{p_{c, 0}}
    \end{equation*}
    and the final conformal mapping satisfying (\ref{norm phi1}) is
    \begin{equation}
        \tilde{\Phi}_1 = \Phi_1 \circ M_{\Phi_1}^{-1}
    \end{equation}

    \begin{algorithm}[H]
    \caption{Normalize $M_1$}
    \label{alg norm phi1}
    \begin{algorithmic}
        \STATE \textbf{Inputs:} $\Phi_1$ and $p_i \in \partial \D$ for $i=1, 2, \cdots, N$, $N=200$, $\epsilon=10^{-5}$.
        \STATE \textbf{Initialize:} Let $k=0$, $p_{i,0} = p_i$, $M_{\Phi_1} = id$ and compute $p_{c,0}= \frac{1}{N} \sum_{i=1}^N p_{i}$. %and $M_{p_{c,0}}(z) = \frac{z-p_{c,0}}{1- \overline{p_{c, 0}}z}$
        \WHILE{$\abs{p_{c,k}} > \epsilon$}
            \STATE Let $F_{p_{c,k}}(z) = \frac{z-p_{c,k}}{1- \overline{p_{c, k}}z}$ and $M_{\Phi_1} = F_{p_{c,k}} \circ M_{\Phi_1}$.%and $\Phi_1 = \Phi_1 \circ F_{p_{c,k}}^{-1}$.
            \STATE Compute $p_{i, k+1} = F_{p_{c,k}}(p_{i, k})$ and $p_{c,k+1} = \frac{1}{N} \sum_{i=1}^N p_{i, k+1}$.
            \STATE Let $k = k+1$.
        \ENDWHILE
        % \STATE Solve equation (\ref{eq}) by Newtown's method and get solution $a \in \D$.
        % \STATE Let $F_a(z) = \frac{z-a}{1- \overline{a}z}$ and $M_{\Phi_1} = F_a \circ M_{\Phi_1}$.% and $\hat{\Phi}_1 = \Phi_1 \circ F^{-1}_{a}$.
        \RETURN Mobi\"us transformation $M_{\Phi_1}$.
    \end{algorithmic}
    \end{algorithm}

    \begin{figure}
        \begin{center}
            \includegraphics[width=8cm]{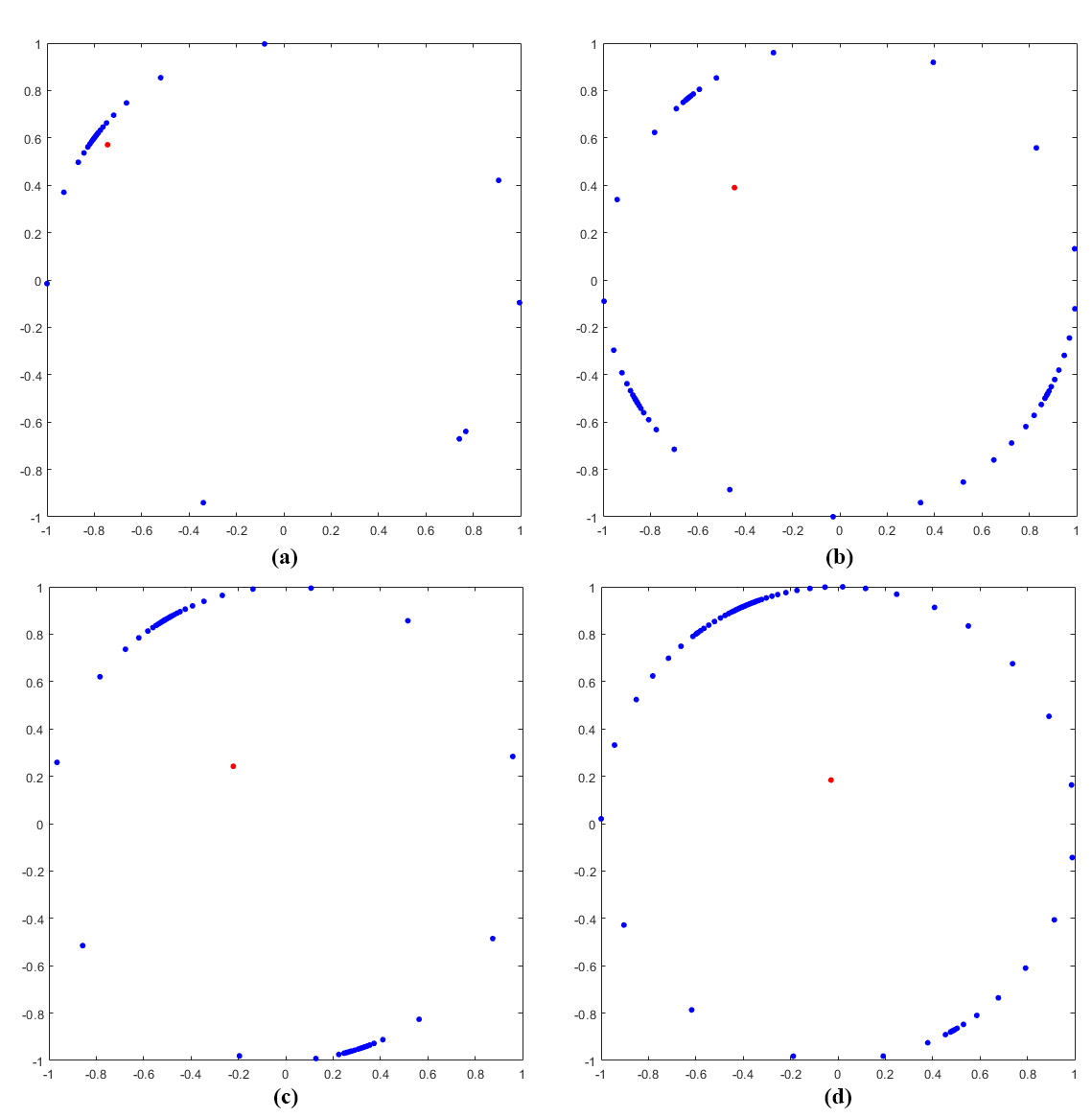}
        \end{center}
        \caption{The iteration of distribution adjustment. The boundary points are blue and their arithmetic center is red in each picture. (a)-(d) The 1st, 3rd, 7th, 12th iteration.}
    \end{figure}
    
    As for the normalization of $M_2$, it's much easier. For requirement (\ref{norm phi2 1}), let $b = \Phi_2^{-1}(\infty)$ and $c = -\frac{1}{\bar{b}}$, from
    \begin{equation*}
        M_{\Phi_2}(\infty) = F_{c}(\infty) = \lim_{z \rightarrow \infty} \frac{z -c}{1 - \bar{c}z} = -\frac{1}{\overline{c}} = b
    \end{equation*}
    we have 
    \begin{equation*}
        \tilde{\Phi}_2(\infty) = \Phi_2 \circ M_{\Phi_2}(\infty) = \Phi_2(b) = \infty.
    \end{equation*}

    \begin{algorithm}[H]
    \caption{Normalize $M_2$}
    \label{alg norm phi2}
    \begin{algorithmic}
        \STATE \textbf{Inputs:} $\Phi_2$ and $p_i \in \partial \D$ for $i=1, 2, \cdots, N$, $N=200$.
        \STATE Compute $b = \Phi_2^{-1}(\infty)$ and $c = -\frac{1}{\overline{b}}$.
        \STATE Let $F_c(z) = \frac{z-c}{1-\overline{c}z}$ and $M_{\Phi_2} = F_{c}$.
        \RETURN Mobi\"us transformation $M_{\Phi_2}$
        % \RETURN Conformal mapping $\hat{\Phi}_2: \D^c \rightarrow \Omega^c$ satisfied (\ref{norm phi2 1}), (\ref{norm phi2 2}) and boundary points $p_i \in \partial \D$.
    \end{algorithmic}
    \end{algorithm}

    \subsection{Harmonic extension}\label{detail harmonic}
        After obtaining the normalized $\tilde{\Phi}_1(z) = g_{\Phi_1}^{-1} \circ M_{\Phi_1}^{-1}(z)$ and $\tilde{\Phi}_2(z) = g_{\Phi_2}^{-1} \circ M_{\Phi_2} (\frac{1}{z})$, the conformal welding can be represented as a series points 
        \begin{equation*}
            (\varphi_i, \omega_i) = \left(\arg(\tilde{\Phi}_2^{-1}(z_i)), \arg(\tilde{\Phi}_1^{-1}(z_i))\right),
        \end{equation*}
        where $\varphi_i, \omega_i \in [0, 2\pi)$. So in fact we should use discrete form Poisson integral to extend $f$ to a harmonic mapping $H$ on the unit disk
        \begin{equation}\label{discrete poisson integral}
            H(re^{i\theta}) = \frac{1}{2\pi} \sum_{j=1}^N \frac{(1-r^2) e^{i \omega_j} \gamma_j}{1 - 2 r cos (\varphi_j - \theta) + r^2} ,
        \end{equation}
        where $\gamma_j = (\varphi_{j} - \varphi_{j-1}) \bmod{ 2\pi}$ and $\varphi_0 = \varphi_N$ and this $\bmod$ can solve some critical value problem, for example, $\varphi_j = 0$ but $\varphi_{j-1} = 6$. For the convenience of computation, we only calculate the value of $H$ on a grid inside unit disk
        \begin{equation}\label{grid G}
            G := \{z=x+i y \, \mid \abs{z} \le 1, x=\frac{j}{M}, y=\frac{k}{M}, j,k=-M,\cdots,M \},
        \end{equation}
        where $M$ is a fixed number and we choose $M=100$ here. Finally, the desired HBS can be generated from $\mu_H$ according to equation (\ref{normaled B}).

\subsection{Summary of the Algorithm}
    The totally algorithm to get HBS is as following.
    \begin{algorithm}[H]
    \caption{Calculate HBS}
    \label{alg all}
    \begin{algorithmic}
        \STATE \textbf{Inputs:} Simply-connected shape $\Omega \subset \C$, $N=200$.
        \STATE Pick clockwise points $z_1, \cdots, z_N \in \partial \Omega$ uniformly.
        \STATE Input $z_1, \cdots, z_N$ to Algorithm \ref{zipper} then get $g_{\Phi_1}: \Omega \rightarrow \D$ and $p_{i,1} \in \partial\D$.
        \STATE Input $z_N, \cdots, z_1$ to Algorithm \ref{zipper} then get $g_{\Phi_2}: \Omega^c \rightarrow \D$ and $p_{i,2} \in \partial\D$.
        \STATE Let $\Phi_1(z) = g_{\Phi_1}^{-1}(z)$ and $\Phi_2(z) = g_{\Phi_2}^{-1}(\frac{1}{z})$.
        \STATE Input $\Phi_1$ and $p_{i,1}$ to Algorithm \ref{alg norm phi1}, then get $M_{\Phi_1}$.
        \STATE Input $\Phi_2$ and $p_{i,2}$ to Algorithm \ref{alg norm phi2}, then get $M_{\Phi_2}$.
        \STATE Let $\tilde{\Phi}_1 = \Phi_1 \circ M_{\Phi_1}^{-1}$ and $\tilde{\Phi}_2 = \Phi_2 \circ M_{\Phi_2}$.
        \STATE Let $f = \tilde{\Phi}_1^{-1} \circ \tilde{\Phi}_2$ and represent it by $(\varphi_i, \omega_i) = (\arg(p_{i,2}), \arg(p_{i,1}))$.
        \STATE Extend $f$ to $H$ on $\D$ by equation (\ref{discrete poisson integral}) on grid $G$.
        \STATE Calculate Beltrami coefficient $\mu_{H}$.
        \STATE Calculate $\theta = \arg \int_\D \mu_H(z) dz$, $\theta' = \arg \int_\D \mu_H(z) / z dz$.
        \IF{$0 \le \theta - \theta'/2 < \pi$}
        \STATE Let $B(z) = e^{i \theta} \mu_{H}(e^{-\frac{1}{2}i\theta} z)$.
        \ELSE
        \STATE Let $B(z) = e^{i \theta} \mu_{H}(-e^{-\frac{1}{2}i\theta} z)$.
        \ENDIF
        \RETURN Harmonic Beltrami signature $B$.
    \end{algorithmic}
    \end{algorithm}

    \subsection{Reconstruction from HBS}
    The proposed HBS is an effective fingerprint of 2D shape and we can also construct the corresponding shape $\Omega$ from given HBS $B$ easily. Different from the theory metioned in theorem \ref{one to one equivalence class}, there is no computational algorithm to find the quasiconformal map $G$ on $\hat{\C}$ directly from its Beltrami coefficient $\mu$ in equation (\ref{reconstructionmu}), so we adopt the following method.
    
    Since $\mu=0$ on $\D^c$, we just focus on the part inside $\D$ and get $F: \D \rightarrow \Omega_0$ by the free boundary quasiconformal deformation method in \cite{choi2020shape}. Since $B$ is the Beltrami coefficient of some harmonic extension $H$ and $\mu_F = B$, such $F$ is also harmonic. Then we pick $z_1, \cdots, z_k$ from $\partial \D$ uniformly and $F(z_1), \cdots, F(z_k) \in \partial \Omega_0$ respectively. After that, the geodesic algorithm \cite{marshall2009lens} welds all $z_i$ and $F(z_i)$, generating conformal maps $g_1 : \Omega_0 \rightarrow \Omega$ and $g_2: \D^c \rightarrow \Omega^c$, where $\Omega = g_1 \circ F (\D)$ and $g_1 \circ F(z_i) = g_2(z_i)$.
    
    Claims that $\Omega$ is the shape we want to reconstruct, up to a rotation, scaling and transformation. There must be some comformal function $C$ maps $\Omega_0$ to $\D$, then $g_1 \circ C^{-1}$ maps $\D$ to $\Omega$ and $C \circ g_1^{-1} \circ g_2$ is conformal welding of $\Omega$. For each $z_i \in \partial \D$, we have
    $$C \circ F(z_i) = C \circ g_1^{-1} \circ g_2(z_i)$$
    and $C \circ F$ is harmonic, so $C \circ F$ is the harmonic extension of conformal welding. Therefore the HBS of $\Omega$ is $\mu_{C \circ F} = B$ and $\Omega$ is the desired shape.

    As metioned in section \ref{geometric implication}, we fix
    \begin{equation*}
        G = \begin{cases}
            g_1 \circ F \text{ on } \D,\\
            g2 \text{ on } \D^c
        \end{cases}
    \end{equation*}
    at $0, 1, \infty$ in order to eliminate the arbitrariness and compare the distance between different reconstructed shapes directly. For above method, it's equivalent with
    \begin{gather*}
        F(0) = 0, F(1) = 1,\\
        g_1(0) = 0, g_1(1) = 1,\\
        g_2(\infty) = \infty.
    \end{gather*}
    These requirements can be achieved naturally in free boundary quasiconformal deformation and geodesic welding algorithm. So the reconstruction algorithm can be described as following.

    \begin{algorithm}[H]
    \caption{reconstruction from HBS}
    \label{reconstruction algo}
    \begin{algorithmic}
        \STATE \textbf{Inputs:} HBS $B: \D \rightarrow \D$, $N=1000$.
        \STATE Let $z_k=e^{2\pi\frac{k}{N}i}\in \partial \D, k=1,\cdots, N$.
        \STATE Reconstruct $F: \D \rightarrow \Omega_1$ from $B$ by free boundary quasiconformal deformation with $F(0)=0$ and $F(1)=1$.
        \STATE Compute $g_1: \Omega_1 \rightarrow \Omega$ and $g_2: \D^c \rightarrow \Omega^c$ by geodesic welding algorithm with $g_1(0) = 0, g_1(1) = 1, g_2(\infty) = \infty$ and $g_2(z_k) = g_1(F(z_k))$.
        \RETURN Reconstructed shape $\Omega = g_1 \circ F(\D)$.
    \end{algorithmic}
    \end{algorithm}

    \section{Experimental result}\label{result}
        In this section, we validate key properties of our proposed Harmonic Beltrami signature, the invariance of under simple transformations and the robustness under small distortion and modification. Then we try to reconstruct shapes from HBS and verify the robustness of reconstruction algorithm. Besides we test the classification performance of HBS.

        Before showing results, what needs to illustrate is that the distance we used to measure the difference of HBS is based on $L^2$ norm as
        \begin{equation}\label{bs dis}
            d(B_1, B_2) = \sqrt{\frac{1}{N} \sum_{i=1}^N \abs{B_1(z_i) - B_2(z_i)}^2},
        \end{equation}
        where $B_1, B_2$ are two different Harmonic Beltrami signatures, $z_i \in \D$ is the face center of triangular mesh $M$ on grid $G$ mentioned in (\ref{grid G}) and $N=62504$ here. All following experiments are implemented in MATLAB R2014a running on 4-way Intel(R) Xeon(R) Gold 6230 processers with 80 cores at 2.10GHz base frequency and 1024 GB RAM with Ubuntu 18.04LTS 64-bit operating system.

    \subsection{Invariance}\label{sec inv}
        We use the dolphin shown in figure \ref{illu of BS} (a) as the original shape, then calculate HBS after scaling, translation and rotation and compare them with the original shape's HBS. The result is displayed in figure \ref{inv}. In this figure, the first column are the sets of boundary points and we remark them as $\Omega_a$ to $\Omega_f$. The second column are the corresponding Harmonic Beltrami signatures $B_a$ to $B_f$. Note that all the Harmonic Beltrami signatures are shown in modulus, i.e. $\abs{B_n}$ for row $n$, and in top view. And the third column(if have) are the histograms of the difference between original shape's Harmonic Beltrami signature, i.e. $\abs{B_n - B_a}$ for row $n$.
        
        Row b and c are about scaling, the shapes are $\Omega_b = \{z \mid z = 1.5 z_a, z_a \in \Omega_a \}$ and $\Omega_c = \{z \mid z = 0.5 z_a, z_a \in \Omega_a \}$ and the distance are $d(B_a, B_b) = 5.5647 \times 10^{-8}$ and $d(B_a, B_c) = 5.3476 \times 10^{-8}$. Row d is about translation, the shape $\Omega_d = \{z \mid z = z_a+100+20i, z_a \in \Omega_a \}$ and the distance is $d(B_a, B_d) = 4.7817 \times 10^{-8}$. Row e is about rotation, the shape is $\Omega_e = \{z \mid z = e^{0.2\pi i} z_a, z_a \in \Omega_a \}$ and the distance is $d(B_a, B_e) = 5.2144 \times 10^{-8}$. Row f is the combination of scaling, translation and rotation, the shape is $\Omega_e = \{z \mid z = 3e^{-0.85\pi i} z_a+350+600i, z_a \in \Omega_a \}$ and the distance is $d(B_a, B_e) = 5.7635 \times 10^{-8}$. These confirm the invariance of HBS and scaling, translation and rotation.

        \begin{figure*}
            \begin{center}
                \includegraphics[width=10cm]{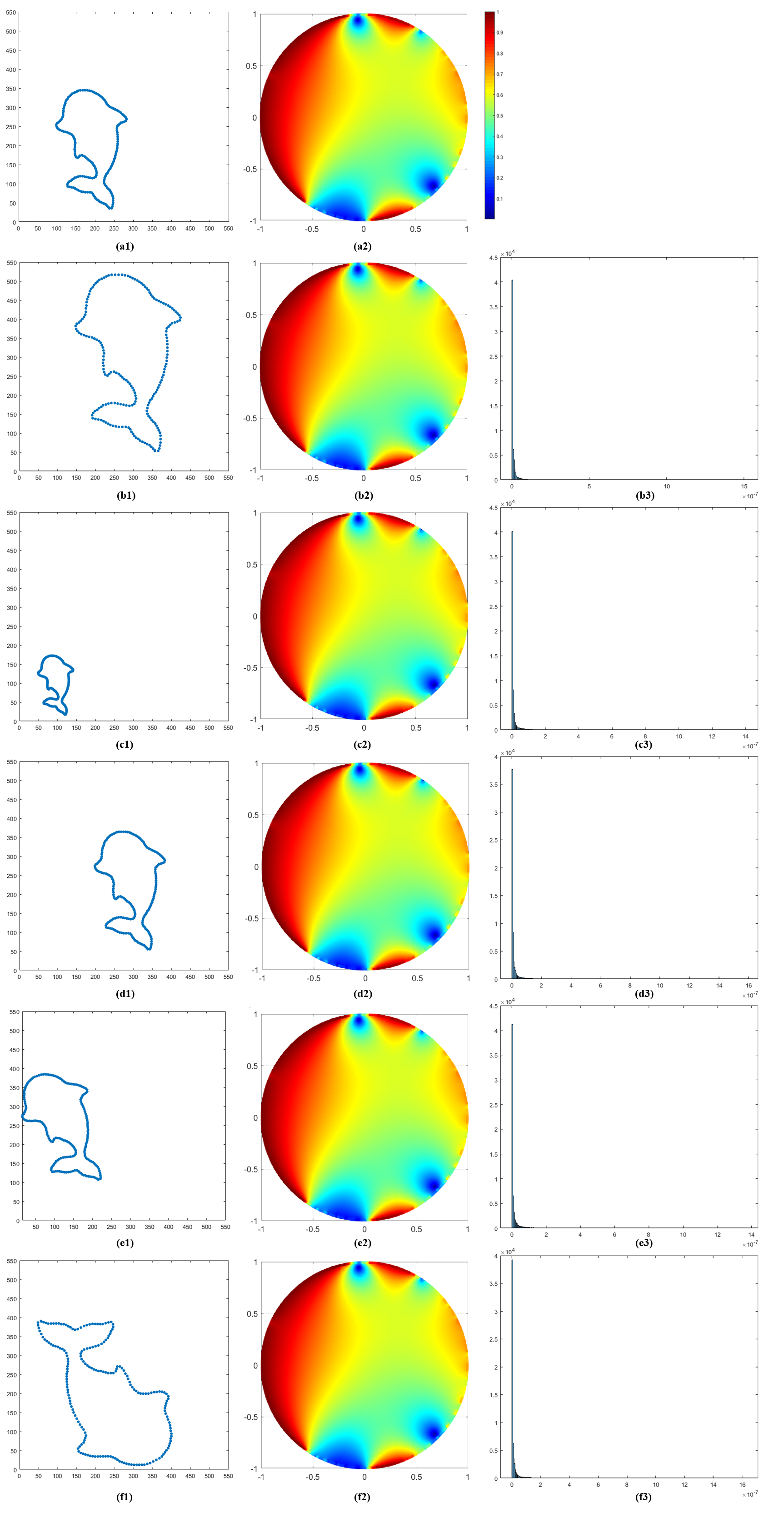}
            \end{center}
            \caption{Harmonic Beltrami signature under scaling, translation and rotation.}
            \label{inv}
        \end{figure*}

    \subsection{Robustness}
        Similar with Section \ref{sec inv}, here we still treat the dolphin as the original shape and modify some small parts of it and figure \ref{robust} is the result. It shows that the proposed signature is robust and stable and will not have a big mutation caused by small disturbance.

        Row g, h and i are result about modification. These shapes are generated by removing or adding something, which is in the red circle. We can see that Harmonic Beltrami signatures have slight differences from $B_a$ but are still similar in general. And this figure also demonstrates that the bigger the modification part is, the more different the Harmonic Beltrami signature is. For example in row i, losing a half of the tail makes the signature has a marked change. Quantitatively, $d(B_a, B_g) = 0.0154$, $d(B_a, B_h) = 0.0461$ and $d(B_a, B_i) = 0.2518$.

        Row j is for distortion. This dolphin is only enlarged in horizontally and becomes fatter, then the $B_j$ moves a little bit and $d(B_a, B_j) = 0.0825$.

        \begin{figure}
            \begin{center}
                \includegraphics[width=10.5cm]{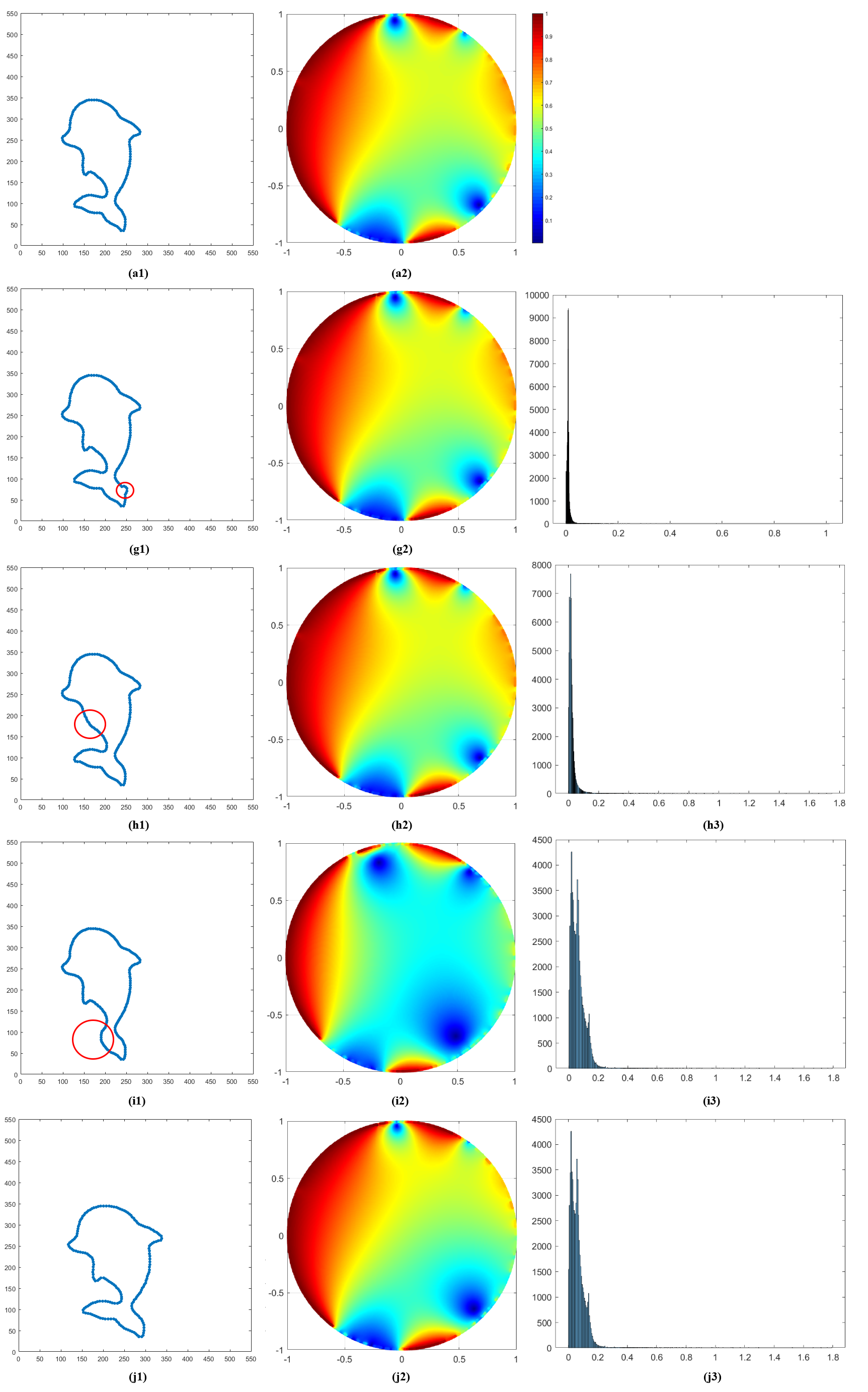}
            \end{center}
            \caption{Harmonic Beltrami signature under small modification}
            \label{robust}
        \end{figure}

    \subsection{Reconstruction from HBS}
        In this experiment, we show our reconstruction results and compare them with their corresponding original shapes. For each row in figure \ref{reconstruction experiment 1}, the left one is original shape, the middle one is HBS and the right one is reconstructed shape. We can find that each reconstructed shape is almost the same with the original one up to a transformation, rotation and scaling, which is owing to no any normalization to original shapes.
        
        Some reconstructed shapes may lose many points where the border is recessed because these points are concentrated in a very small range in HBS. If we use a bigger $N$ in algorithm \ref{reconstruction algo}, The missing part in reconstructed shape will get smaller, as shown in figure \ref{reconstruction experiment 2}.

        \begin{figure}
            \begin{center}
                \includegraphics[width=10cm]{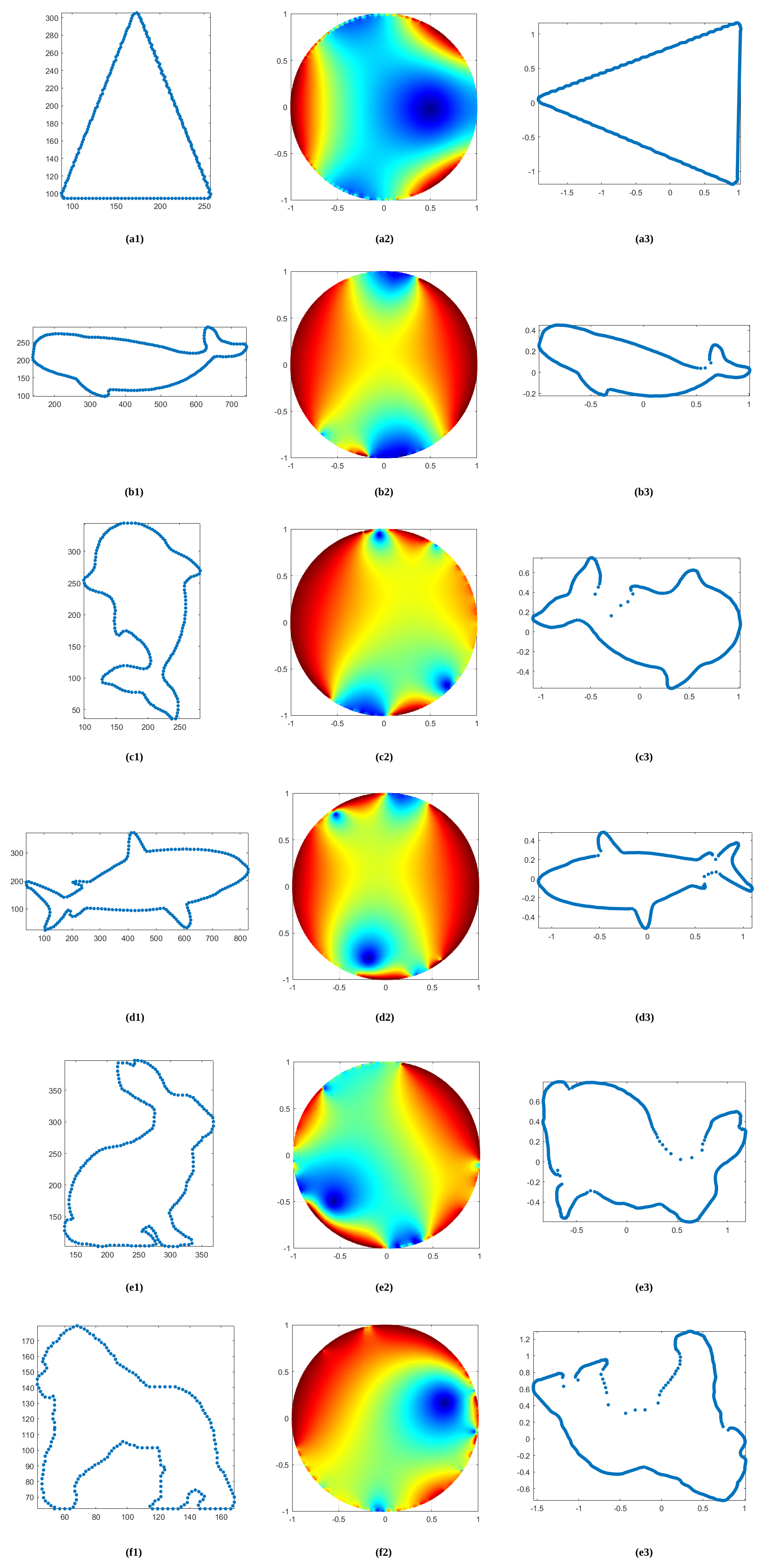}
            \end{center}
            \caption{Reconstruction from HBS}
            \label{reconstruction experiment 1}
        \end{figure}

        \begin{figure}
            \begin{center}
                \includegraphics[width=10cm]{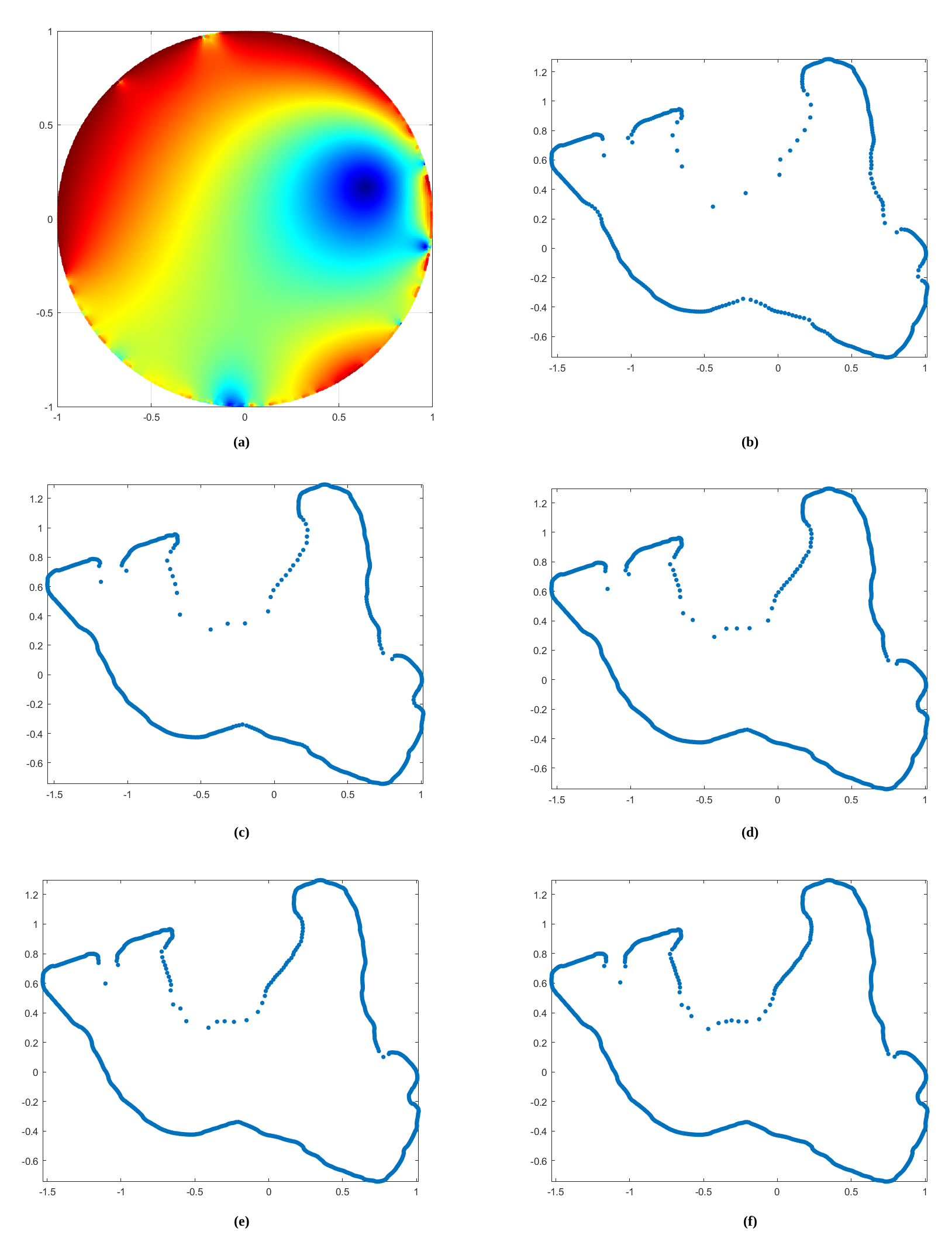}
            \end{center}
            \caption{(a) is the HBS and (b)-(f) are reconstructed shapes with different $N$, which are 500, 1000, 1500, 2000 and 2500 respectively.}
            \label{reconstruction experiment 2}
        \end{figure}

    \subsection{Robustness of reconstruction}
        Here we modify HBS directly by some function and then reconstruct.  The basic HBS $B_0$ is from figure \ref{reconstruction experiment 1} (b2) and
        \[
            B_k(z) = \abs{B_0(z)}^k B_0(z).
        \]
        $\Omega_k$ is reconstructed from $B_k$ as algorithm \ref{reconstruction algo}. In figure \ref{modify HBS}, the reconstructed whales become fatter and fatter when $k$ increases, but they still retain some important features.
        
        In order to describe this similarity more accurately, the distance of HBS from basic HBS is calculated by $L_\infty$ norm
        $$d_{HBS}(B, B_0) = \norm{B_0 - B}_\infty,$$
        and distance between reconstructed shape and original shape is computed as equation (\ref{shape distance})
        $$d_{\Omega}(\Omega, \Omega_0) =  \frac{1}{2}\left(\max_{q\in\partial \Omega_0} \min_{p\in\partial \Omega} ||p-q|| + \max_{p\in\partial \Omega} \min_{q\in\partial \Omega_0} ||p-q||\right).$$
        Except the $B_k$ metioned above, we also try many different $k$ and some other methods to edit the basic HBS $B_0$, like %$\tilde{B}_k(z) = e^{ik} B_0(z)$, 
        $\tilde{B}_k(z) = e^{i(1-\abs{B_0(z)})k}B_0(z)$ and so on. All these $d_{HBS}$ and $d_\Omega$ are shown in figure \ref{d_hbs and d_omega}, which indicates they are linearly related and support our theorem \ref{reconstruction robust theorem} strongly.

        \begin{figure}
            \begin{center}
                \includegraphics[height=22cm]{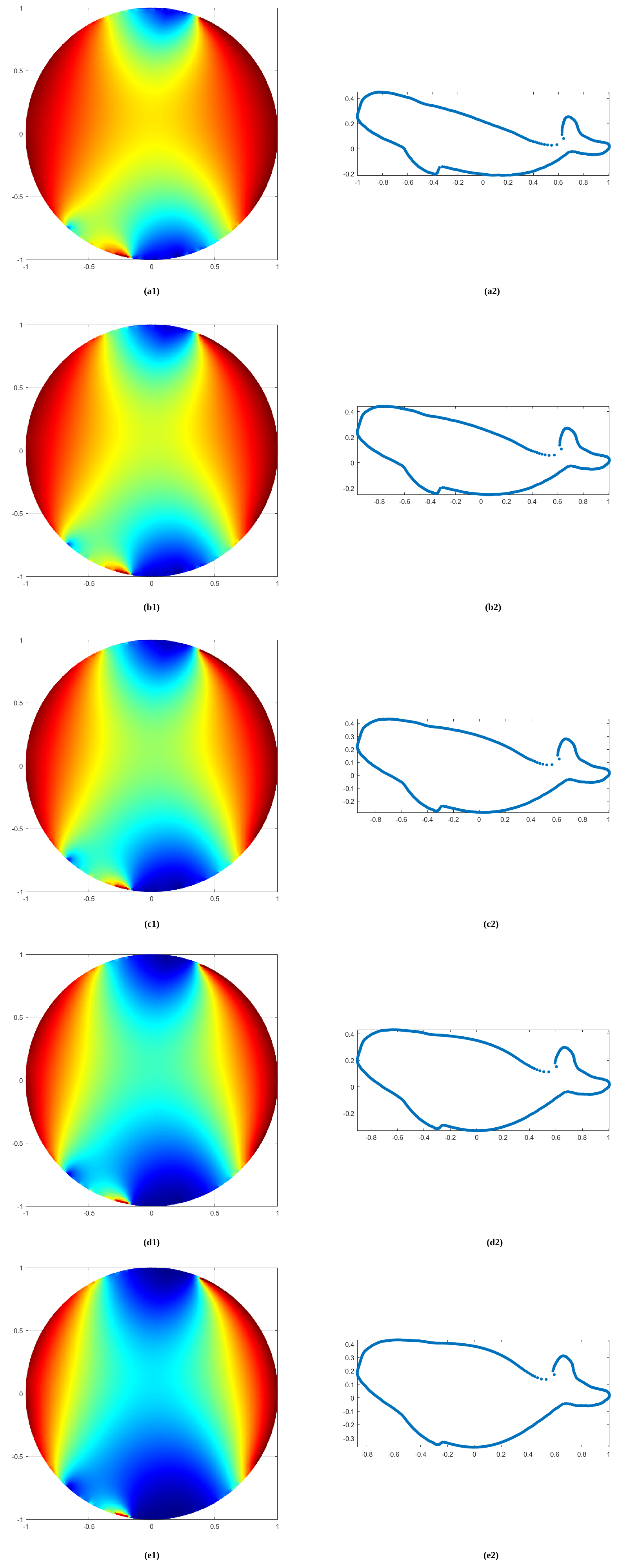}
            \end{center}
            \caption{Different $B_k$ and their reconstructed shapes, where the $k$ is -0.1, 0.2, 0.5, 1 and 1.5 for (a)-(e)}
            \label{modify HBS}
        \end{figure}

        \begin{figure}
            \begin{center}
                \includegraphics[width=8cm]{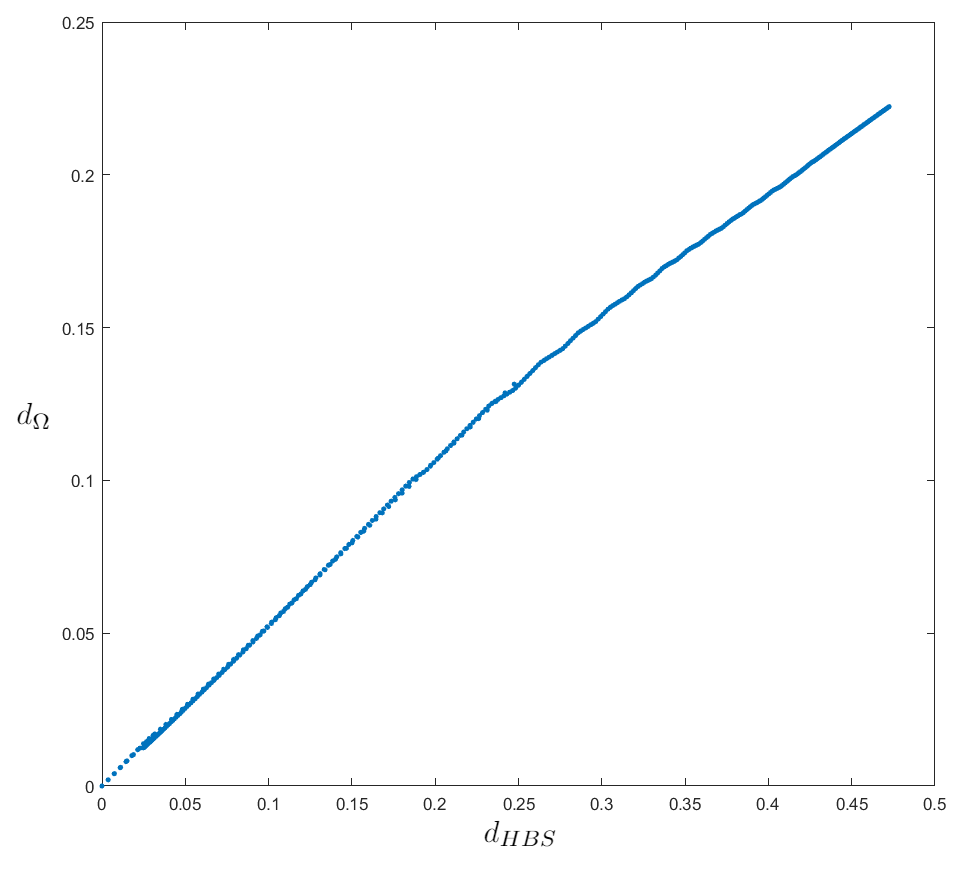}
            \end{center}
            \caption{$d_{HBS}$ and $d_\Omega$, each small blue point here means a pair of HBS and its reconstructed shape.}
            \label{d_hbs and d_omega}
        \end{figure}

    \subsection{Classification with HBS}
        Above properties ensure the proposed signature having the ability to reflect some stable features of given shape, but another much more important thing people concerned is that whether it can distinguish a shape from many different kinds of shapes and classify it correctly. A good signature should keep the similarity in the same kind of shapes and be significantly different for different kinds of shapes.

        To compare the classification performance, we also calculate the conformal welding of all the shapes directly to classify, and the distance is defined as
        \begin{equation}\label{welding dis}
            d_c(f_1, f_2) = \sqrt{\frac{1}{N} \sum_{i=1}^N \abs{f_1(z_i) - f_2(z_i)}^2},
        \end{equation}
        where $f_1, f_2$ are two different conformal welding, $z_i \in \partial \D$ and $N=1000$ here. Note that the conformal welding is not unique, but we can normalize it by
        \begin{equation}
        \begin{split}\label{norm to cw}
            \sum_{i=1}^N \Phi_1^{-1}(z_i) = 0,\\
            \Phi_2(\infty) = \infty,\\
            \Phi_2'(\infty) >= 0,\\
            f(1) = 1.
        \end{split}
        \end{equation}
        
        The first requirement is what we proposed in section \ref{norm_to_m1} and the other three are from \cite{sharon20062d}. After above normalizations, we can fix a unique conformal welding from the given shape.

        We prepare 3 images for 3 kinds of animals, fish, giraffe and elephant, so 9 images in total. From figure \ref{classification images}, we can find that each class share similar HBS but their conformal weldings look different. 
        
        Figure \ref{dis matrix} shows the intraclass distance of HBS is always smaller than interclass distance. But for conformal welding, the data is messy, for example, elephant 2 thinks itself is very different from other two elements but looks most like fish 3. After multidimensional scaling(MDS), we can maps all these 9 shapes to points on 2D plane as figure \ref{mds1}, where the HBS shows powerful classification ability.
        
        Actually, a closer look on figure \ref{classification images} shows the biggest difference between conformal weldings of the same class is just a translation. For example, if the conformal welding of giraffe 3 moves a little right up (mod by 2$\pi$), it will become very similar with other two giraffes. Further research tells us this difference is the result of the third requirement of (\ref{norm to cw}), which essentially looks for a starting point on shape boundary to be mapped to $(0, 0)$ in conformal welding. However, the starting point determined in this way is sensitive, so the corresponding conformal welding is unstable.

        \begin{figure}
            \begin{center}
                \includegraphics[width=12cm]{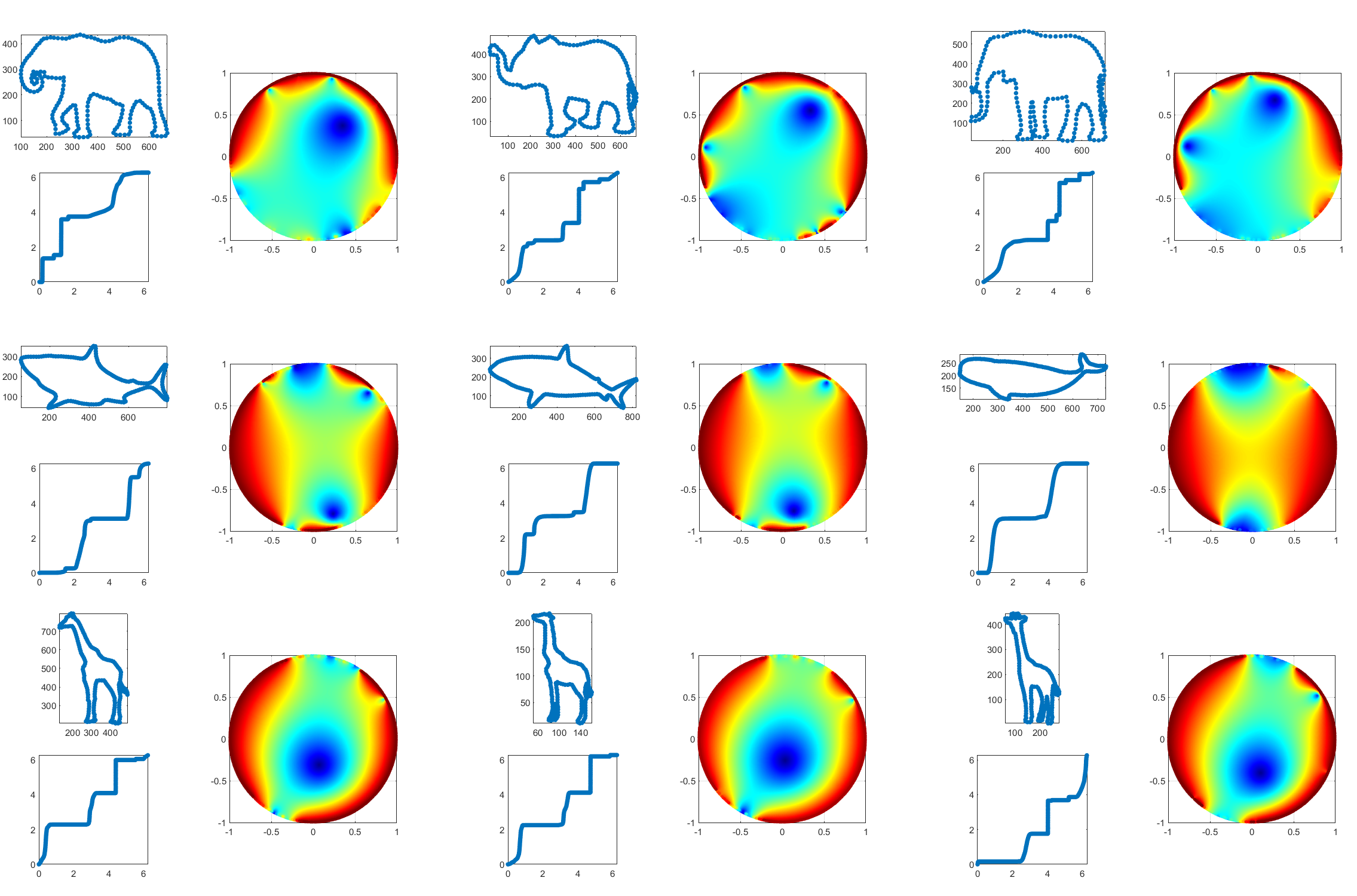}
            \end{center}
            \caption{These 3 rows are elephant, fish and giraffe. In each subfigure, the top left is the input shape, bottom left is the conformal welding and the right is Harmonic Beltrami signature.}
            \label{classification images}
        \end{figure}

        \begin{figure}
            \begin{center}
                \includegraphics[width=13cm]{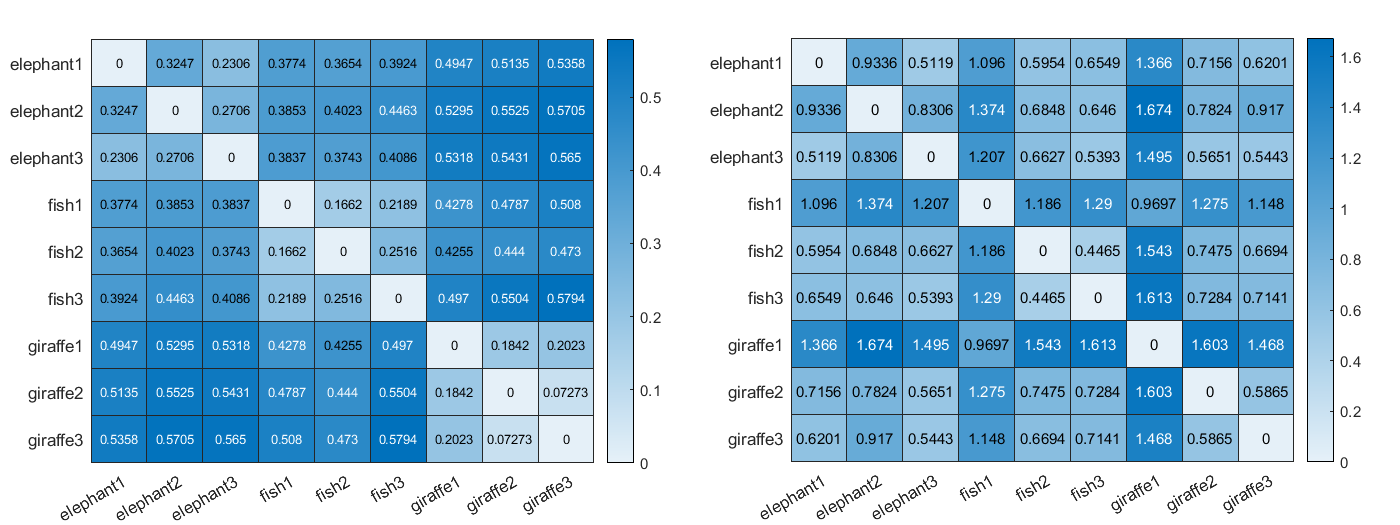}
            \end{center}
            \caption{(a) The distance matrix of Harmonic Beltrami signatures of above 9 shapes by equation (\ref{bs dis}); (b) The distance matrix of conformal weldings by equation (\ref{welding dis}).}
            \label{dis matrix}
        \end{figure}

        \begin{figure}
            \begin{center}
                \includegraphics[width=13cm]{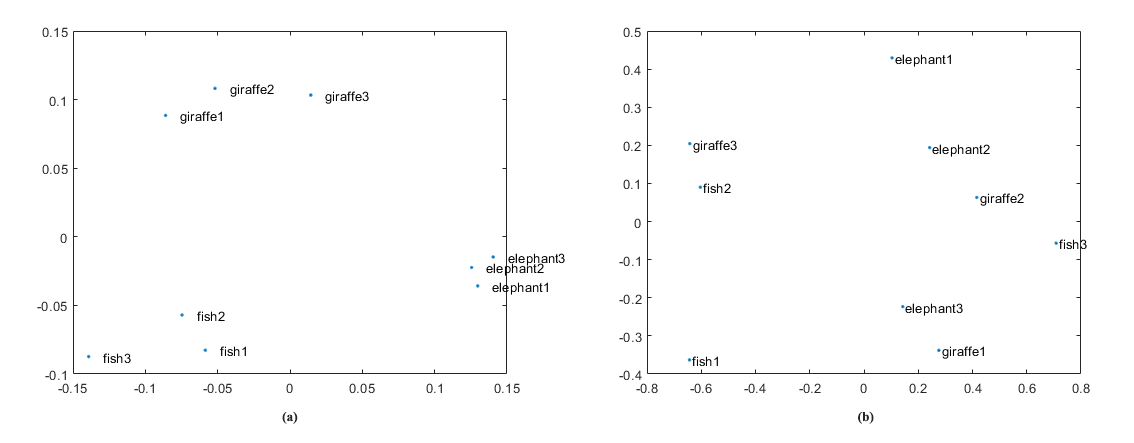}
            \end{center}
            \caption{(a) The MDS result of Harmonic Beltrami signature; (b) The MDS result of conformal welding.}
            \label{mds1}
        \end{figure}

    \subsection{Multi-class classification}\label{more class}
        In this experiment, we enlarge the amount of images to 58 in 7 different classes, which are camel, deer, dog, elephant, giraffe, gorilla and rabbit. All these shapes are in figure \ref{more class all}.
        
        We compare the classification performance of our HBS with conformal welding, shape context \cite{belongie2002shape} and boundary moments \cite{gupta1987267}. The distance of HBS and conformal welding are equation (\ref{bs dis}) and (\ref{welding dis}) respectively. Then the shape context distance defined in \cite{belongie2006matching} can measure the difference between shape contexts. And $L_2$ norm is a suitable distance for boundary moments.
        
        For each algorithm, the distances of any two shapes are calculated and form a distance matrix as last experiment, then MDS remaps these shapes to 2D plane accordingly and $k$-medoids method is used to cluster these points to 7 classes. The MDS and clustering results are displayed in figure \ref{mds2} and the classification accuracy can be found in table \ref{acc table}.
        
        These results demonstrate the multi-class classification performance of our proposed HBS is much better than boundary moments and conformal welding. As for shape context, Although its classification accuracy of is is very close to our HBS, we can find from $(a1)$ and $(b1)$ of figure \ref{mds2} that our HBS can separate different types more apart and keep relatively clear boundaries while shape context concentrates dogs, gorillas, camel and elephants in a very small range. By the way, the shape context distance is much more complicated and it takes more 5 hours(18377.49 seconds) to calculate distances between each 2 of these 58 shapes but our HBS only need 0.23 seconds.

        \begin{figure}
            \begin{center}
                \includegraphics[width=10cm]{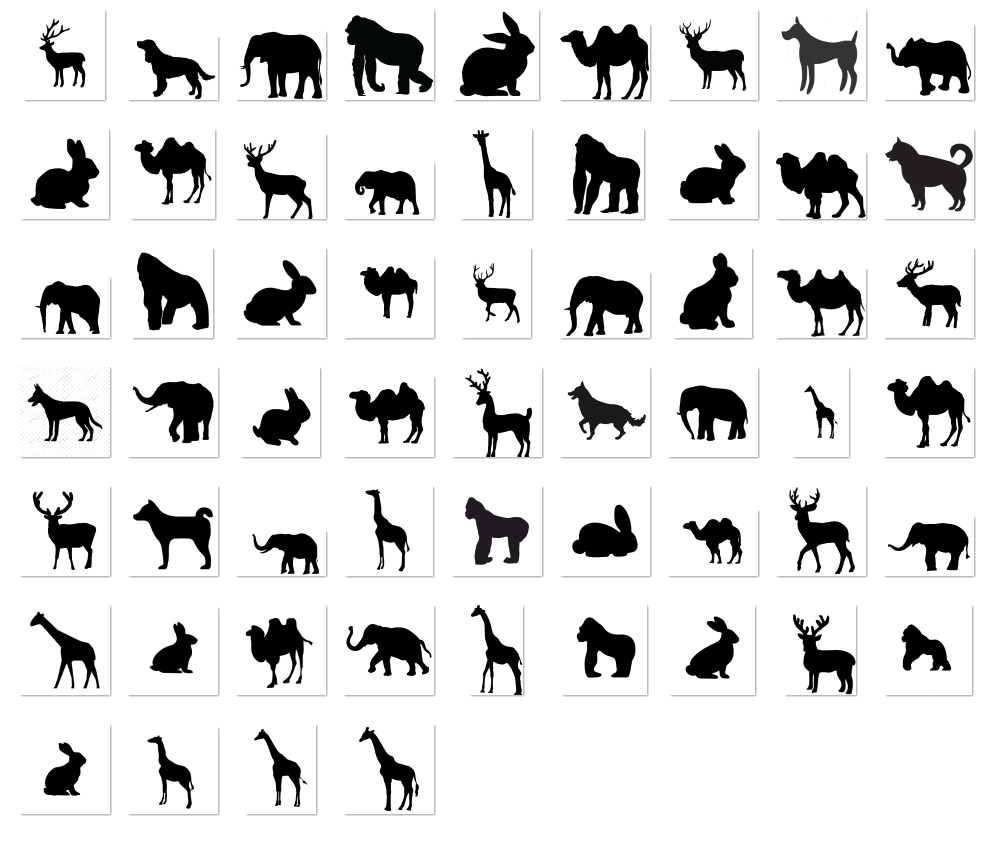}
            \end{center}
            \caption{All 58 shapes within 7 classes used in experiment \ref{more class}}
            \label{more class all}
        \end{figure}

        \begin{figure*}
            \begin{center}
                \includegraphics[width=\textwidth]{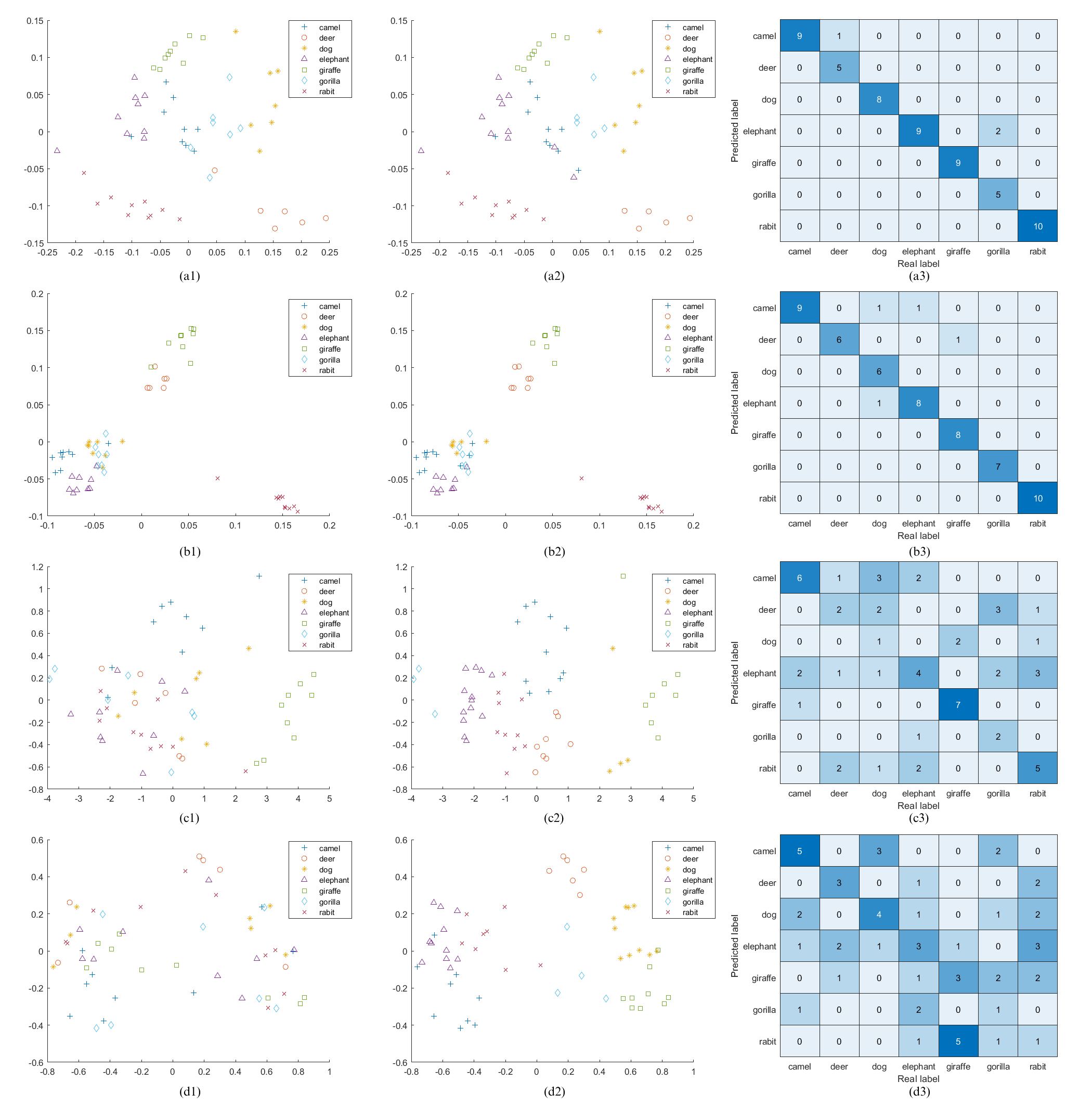}
            \end{center}
            \caption{Row (a) is our proposed HBS, row (b) is shape context, row (c) is boundary moments and row (d) is conformal welding.For each row, the left is MDS result, the middle is $k$-medoids classification result and the right is confusion matrix of classification.}
            \label{mds2}
        \end{figure*}

        \begin{table}[]
        \centering
        \begin{tabular}{|c|c|c|c|c|}
        \hline
        Algorithm & our HBS & Shape Context & Boundary Moments & Conformal Welding \\ \hline
        Accuracy  & 94.83\% & 93.10\% & 46.55\% & 34.48\% \\ \hline
        \end{tabular}
        \caption{Multi-class classification accuracy for our HBS, shape context, boundary moments and conformal welding.}
        \label{acc table}
        \end{table}

    \section{Conclusion}\label{conclusion}
        In this paper, we propose a novel shape representation for 2D bounded simply-connected objects called Harmonic Beltrami signature. The proposed signature is based on conformal welding but overcome a key shortcoming and it can be uniquely determined by the given shape. What's more exciting is that the proposed representation is invariant under scaling, translation and rotation. For slight deformation and distortion, HBS keeps robust and only changes within a reasonable small range. Conversely, if two HBS are alike, their corresponding domains are similar. Therefore, we have every reason to believe the it does have ability to represent some invariant geometrical features. The experimental results also confirm that the HBS has excellent performance in multi-classification tasks.

        Although our work has achieved relatively good results, the proposed representation still have some limitations. Firstly, the HBS is only applicable to simply-connected shapes currently, but as a matter of fact, multi-connected images are the majority in the real world. So we are eager for a feasible method to extend our HBS to multi-connected situation. Secondly, the traditional algorithm to compute the Beltrami coefficient is inevitably dependent on triangular mesh, which consumes a lot of time. Therefore, a fast algorithm to obtain this signature avoiding dense mesh is of high priority in our future work.

        In summary, we will focus on three major directions in the future. One is that the deeper meaning of HBS is worth digging and then a multi-connected version of representation based on this work can be proposed. Another is that if the HBS contains some geometrical features of shapes, we can also extract them directly from images and generate the HBS again. Hence the deep learning theory may help us to compute this signature from given images immediately, which is very likely to improve algorithm speed performance greatly. A third direction is this representation can be used in more applications like segmentation, registration and so on.
%\begin{acknowledgements}
%\end{acknowledgements}


% \biboptions{sort&compress}
% \bibliographystyle{spmpsci}      % mathematics and physical sciences
\bibliographystyle{siamplain}       % APS-like style for physics
\bibliography{cite}
% \begin{thebibliography}{}
% %
% % and use \bibitem to create references. Consult the Instructions
% % for authors for reference list style.
% % Format for Journal Reference
% Author, Article title, Journal, Volume, page numbers (year)
% Author, Book title, page numbers. Publisher, place (year)
% \end{thebibliography}

\end{document}

%% file: snippets.tex
\newcommand{\R}{\mathbb{R}}

\newcommand{\Z}{\mathbb{Z}}

\newcommand{\C}{\mathbb{C}}
\newcommand{\D}{\mathbb{D}}
\newcommand{\abs}[1]{\left\vert #1 \right\vert}
\newcommand{\norm}[1]{\left\Vert #1 \right\Vert}
\renewcommand{\Re}{\text{Re}}

\newcommand{\Part}[2]{\frac{\partial #1}{\partial #2}}
\newcommand{\etal}{\textit{et al. }}